\documentclass[final, 12pt]{style}

\title[Exploration in model-based RL]{\huge Exploration in Model-based Reinforcement Learning with Randomized Reward}
\usepackage{times}

\usepackage{booktabs}
\makeatletter
\let\Ginclude@graphics\@org@Ginclude@graphics
\makeatother

\usepackage{enumerate}
\usepackage[OT1]{fontenc}
\usepackage{booktabs}
\usepackage{algorithm}
\usepackage{algorithmic}
\usepackage{dsfont}

\let\tilde\widetilde

\def\given{{\,|\,}}

\newcommand{\cA}{\mathcal{A}}

\newcommand{\cC}{\mathcal{C}}
\newcommand{\cD}{\mathcal{D}}

\newcommand{\cG}{\mathcal{G}}

\newcommand{\cO}{\mathcal{O}}

\newcommand{\cS}{{\mathcal{S}}}

\newcommand{\EE}{\mathbb{E}}

\newcommand{\PP}{\mathbb{P}}

\newcommand{\RR}{\mathbb{R}}

\newcommand{\argmin}{\mathop{\mathrm{argmin}}}
\newcommand{\argmax}{\mathop{\mathrm{argmax}}}

\newcommand{\tr}{\mathop{\mathrm{tr}}}

\newtheorem{condition}[theorem]{Condition}

\ifx\BlackBox\undefined
\newcommand{\BlackBox}{\rule{1.5ex}{1.5ex}}  
\fi

\ifx\QED\undefined
\def\QED{~\rule[-1pt]{5pt}{5pt}\par\medskip}
\fi

\ifx\example\undefined
\newtheorem{example}[theorem]{Example}
\fi
\ifx\property\undefined
\newtheorem{property}{Property}
\fi
\ifx\lemma\undefined
\newtheorem{lemma}[theorem]{Lemma}
\fi
\ifx\proposition\undefined
\newtheorem{proposition}[theorem]{Proposition}
\fi
\ifx\remark\undefined
\newtheorem{remark}[theorem]{Remark}
\fi
\ifx\corollary\undefined
\newtheorem{corollary}[theorem]{Corollary}
\fi
\ifx\definition\undefined
\newtheorem{definition}[theorem]{Definition}
\fi
\ifx\conjecture\undefined
\newtheorem{conjecture}[theorem]{Conjecture}
\fi
\ifx\fact\undefined
\newtheorem{fact}[theorem]{Fact}
\fi
\ifx\claim\undefined
\newtheorem{claim}[theorem]{Claim}
\fi
\ifx\assumption\undefined
\newtheorem{assumption}[theorem]{Assumption}
\fi
\ifx\cond\undefined
\newtheorem{cond}[theorem]{Condition}
\fi

\makeatletter
\newcommand{\raisemath}[1]{\mathpalette{\raisem@th{#1}}}
\newcommand{\raisem@th}[3]{\raisebox{#1}{$#2#3$}}
\makeatother

\let\oldabstract\abstract
\let\oldendabstract\endabstract
\makeatletter
\renewenvironment{abstract}
{%
               {\list{}{\addtolength{\leftmargin}{1.75em} 
                       \listparindent 1em%
                        \itemindent    \listparindent%
                        \rightmargin   \leftmargin%
                        \parsep        \z@ \@plus\p@}%
                \item\relax}%
               {\endlist}%
\oldabstract}
{\oldendabstract}
\makeatother

\def\given{\,|\,}
\def\biggiven{\,\big{|}\,}
\def\tr{\mathop{\text{tr}}\kern.2ex}

\def\good{\mathop{\text{good}}}

\long\def\comment#1{}

\def\tr{\mathop{\text{Tr}}}

\def\cS{{\mathcal{S}}}

\newcommand{\bel}{\begin{eqnarray}\label}
\newcommand{\eel}{\end{eqnarray}}
\newcommand{\bes}{\begin{eqnarray*}}
\newcommand{\ees}{\end{eqnarray*}}

\newcommand{\eps}{\epsilon}

\usepackage{mathrsfs}

\def\algoname{\texttt{PlanEx}~}
\def\algonospace{\texttt{PlanEx}~}
\def\##1\#{\begin{align}#1\end{align}}
\def\$#1\${\begin{align*}#1\end{align*}}
\usepackage{enumitem}

\author{\vspace{0.5in}\\
. \hspace{2.1in} Lingxiao Wang, \  Ping Li\\\\
. \hspace{2in} Cognitive Computing Lab\\
. \hspace{2.5in} Baidu Research\\
. \hspace{1.6in} 10900 NE 8th St, Bellevue, WA 98004, USA\\
}

\begin{document}

\maketitle

\begin{abstract}\vspace{0.2in}

\noindent\footnote{This manuscript was  completed in September 2021 while both authors worked at Baidu Cognitive Computing Lab.   }Model-based Reinforcement Learning (MBRL) has been widely adapted due to its sample efficiency. However, existing worst-case regret analysis typically requires optimistic planning, which is not realistic in general. In contrast, motivated by the theory, empirical study utilizes ensemble of models, which achieve state-of-the-art performance on various testing environments. Such deviation between theory and empirical study leads us to question whether randomized model ensemble guarantee optimism, and hence the optimal worst-case regret? This paper partially answers such question from the perspective of reward randomization, a scarcely explored direction of exploration with MBRL. We show that under the kernelized linear regulator (KNR) model, reward randomization guarantees a partial optimism, which further yields a  near-optimal worst-case regret in terms of the number of interactions. We further extend our theory to generalized function approximation and identified conditions for reward randomization to attain provably efficient exploration.~Correspondingly, we propose concrete examples of efficient reward randomization. To the best of our knowledge, our analysis establishes the first worst-case regret analysis on randomized MBRL with function approximation.

\end{abstract}

\newpage

\section{Introduction}

Reinforcement learning (RL)~\citep{sutton2018reinforcement} aims to learn the optimal policy by iteratively interacting with the environment. Model-based reinforcement learning (MBRL)~\citep{osband2014model, luo2019algorithmic, ha2018world, luo2019algorithmic, sun2019model, kaiser2020model, ayoub2020model, kakade2020information} achieves such a goal by fitting the environment from the observation and obtaining the policy from the fitted environment. Incorporated with deep learning, MBRL has achieved tremendous success in real-world tasks, including video games~\citep{ha2018world, kaiser2020model} and control tasks~\citep{watter2015embed, williams2015model, chua2018deep, hafner2019dream, song2021pc}.

A key factor to the success of MBRL is sample efficiency. In terms of the theoretical analysis, such sample efficiency is characterized by the regret analysis of MBRL. Previous analysis suggests that when incorporated with exploration strategies, MBRL enjoys a near-optimal $\tilde\cO(\sqrt{T})$ regret~\citep{jaksch2010near, ayoub2020model, kakade2020information}, where $T$ is the total number of interactions with the environment. However, previous provably efficient exploration typically utilizes optimistic planning~\citep{jaksch2010near,  luo2019algorithmic, ayoub2020model, kakade2020information}. Such exploration strategy requires (i) identifying a confidence set of models $\cD$, which captures the uncertainty in model estimation, and then (ii) conducting optimistic planning by searching for the maximal policy among all possible models within $\cD$. The key to the success of optimistic planning is \textit{optimism under the face of the uncertainty} principle~\citep{jaksch2010near}. Intuitively, optimistic planning encourages the agent to explore less visited areas, hence enhancing the sample complexity of the corresponding RL algorithm. While step (i) is realizable with ensemble techniques, step (ii) is in general impossible to implement, as it requires solving an optimization problem over a possibly continuous space of models $\cD$. As an alternative, previous empirical study~\citep{chua2018deep, pathak2017curiosity, pathak2019self} typically borrows the idea from optimistic planning and the study of Thompson sampling (TS) based algorithm~\citep{osband2014model}. A common empirical approach is to utilize model ensembles to capture the uncertainty of model estimations. Such ensembles are further utilized in planning through TS~\citep{chua2018deep} or bonus construction~\citep{pathak2017curiosity, pathak2019self}. Unlike optimistic planning, such approaches typically do not have a worst-case regret guarantee. Nevertheless, they attain state-of-the-art performance in various testing environments. Such deviation from theory and practice motivates us to propose this question:\\

\noindent\textit{Does randomized model ensemble guarantees optimism, and hence the optimal worst-case regret? }\\

In this paper, we provide a partial solution to the above question from reward randomization, a relatively less studied method for exploration in MBRL. We initiate our analysis under the \textit{kernelized linear regulator} (KNR) transition model~\citep{mania2022active, kakade2020information, song2021pc} and known reward functions. We propose \algonospace, which conducts exploration~by iteratively planning with the fitted transition model and a randomized reward function. We further show that \algoname attains the near-optimal $\tilde\cO(\sqrt{T})$ regret. A key observation of reward randomization~is a notion of \textit{partial optimism}~\citep{russo2019worst, zanette2020frequentist}, which ensures that a sufficient amount of interactions are devoted to exploration under the optimism principle. Motivated by the analysis under the KNR transition model, we extend \algoname to general function approximation with calibrated model~\citep{curi2020efficient, kidambi2021mobile} and propose a generic design principle of reward randomization. We further propose concrete examples of valid reward randomization and demonstrate the effectiveness of reward randomization  theoretically. We highlight that the proposed reward randomization method can be easily implemented based on the model ensembles. In addition, the reward randomization is highly modular and can be incorporated with various SOTA baselines.

\vspace{0.1in}
\noindent{\bf Contribution.} Our work provides a partial solution to the question we raised. Specifically, we investigate reward randomization and propose \algonospace. Our contributions are as follows.
\begin{itemize}
\item We propose \algonospace, a novel exploration algorithm for MBPO with worst-case regret guarantee that is realizable with general function parameterizations.
\item We show that \algoname~has near-optimal worst-case regret under the KNR dynamics.
\end{itemize}
To the best of our knowledge, our analysis establishes the first worst-case regret analysis on randomized MBRL with function approximation.

\vspace{0.1in}
\noindent{\bf Related Work.}
Our work is closely related to the regret analysis of MBRL and online control problem~\citep{osband2014model, luo2019algorithmic, sun2019model, lu2019information, ayoub2020model, kakade2020information, curi2020efficient, agarwal2020flambe, song2021pc}.~\cite{ayoub2020model} propose the value-target regression (VTR) algorithm, which focuses on the aspects of the transition model that are relevant to RL.~\cite{agarwal2020flambe} propose FLAMBE, a provably efficient MBRL algorithm under the linear MDP setting~\citep{jin2020provably, yang2019sample}.~\cite{kakade2020information} propose LC3, an online control algorithm under the KNR dynamics~\citep{mania2022active} that attains the optimal worst-case regret. Both~\cite{ayoub2020model} and~\cite{kakade2020information} utilizes optimistic planning for exploration, which is in general intractable. In contrast, we utilize reward randomization for exploration, which also attains the optimal worst-case regret and is highly tractable. To attain tractable optimistic planning,~\cite{curi2020efficient} design HUCRL, which introduces an extra state deviation variable in the optimization for planning. In contrast, planning with randomized reward does not introduce extra variable in optimization.~\cite{luo2019algorithmic} optimizes a lower bound of value functions, which avoids explicit uncertain quantification. Recent works also utilizes reward bonus to attain optimistic planning~\citep{kidambi2021mobile, song2021pc}.~\cite{song2021pc} propose PC-MLP, which constructs bonus by estimating the policy cover~\citep{agarwal2020pc} and is computationally tractable. As a comparison, PC-MLP requires extra sampling to estimate the covariate matrix of policy cover. As a consequence, PC-MLP does not attain the optimal $\tilde\cO(\sqrt{T})$-regret. In contrast, \algoname does not require extra sampling to construct the bonus and can achieve the $\tilde\cO(\sqrt{T})$-regret. In addition, to attain tractable realization, PC-MLP utilizes different feature in model fitting and bonus construction, which is inconsistent with the theoretical analysis. In contrast, the implementation of \algoname is consistent with the theoretical analysis under the calibrated model assumption. Previous work also study efficient model-free RL exploration algorithms with function approximation. See, e.g.,~\cite{jiang2017contextual, jin2020provably, du2020good, wang2020reward, cai2020provably, agarwal2020pc, modi2021model} and references therein for this line of research.

Our analysis is inspired by the recent progress in worst-case regret analysis of randomized RL algorithms~\citep{russo2019worst, pacchiano2020optimism,zanette2020frequentist, ishfaq2021randomized}. Our optimism analysis is inspired by~\cite{russo2019worst} and~\cite{zanette2020frequentist}.~\cite{russo2019worst} propose the first worst-case regret analysis to the randomized least-squares value iteration (RLSVI) algorithm under the tabular setting.~\cite{zanette2020frequentist} extend the analysis of RLSVI to truncated linear function approximations under the general state space.~\cite{ishfaq2021randomized} analyze the randomized Q-learning with both linear function approximation and general function approximation. In contrast, we focus on the randomized MBRL algorithm.~\cite{pacchiano2020optimism} analyze the worst-case regret of MBRL with both reward and transition randomization. We remark that both~\cite{pacchiano2020optimism} and~\cite{ishfaq2021randomized} require drawing multiple samples for each state-action pair in planning and further maximizing over all randomized reward functions in planning. In contrast, we only require one sample at each time step, and do not need further maximization. In addition,~\cite{pacchiano2020optimism} focus on the tabular setting with finite state and action spaces, whereas we consider the generic setting with function approximation.


\section{Background}

\subsection{Reinforcement Learning}
In this paper, we model the environment by an episodic MDP $(\cS, \cA, H, \{r_h\}_{h\in[H]}, \PP)$. Here $\cS$ and $\cA$ are the state spaces, $H$ is the length of episodes, $r_h: \cS\times\cA\mapsto [0, 1]$ is the bounded reward function for $h\in[H]$, and $\PP$ is the transition kernel, which defines the transition probability $s_{h+1}\sim\PP(\cdot\given s_h, a_h)$ for all $h\in[H]$ and $(s_h, a_h)\in\cS\times\cA$.
\vspace{0.1in}
\noindent{\bf Interaction Procedure.} An agent with a set of policies $\{\pi_h\}_{h\in[H]}$ interacts with such environment as follows. The agent starts from a fixed initial state $s_1\in\cS$. Iteratively, upon reaching the state $s_h\in\cS$, the agent takes the action $a_h = \pi_h(s_h)$. The agent then receives the reward $r_h(s_h, a_h)$. The environment transits into the next state $s_{h+1}$ according to the probability $\PP(\cdot\given s_h, a_h)$. The process ends when the agent reaches the state $s_{H+1}$.

To describe the expected cumulative reward, for each policy $\pi = \{\pi_h\}_{h\in[H]}$, we introduce the action-value functions $\{Q^\pi_h\}_{h\in[H]}$ defined as follows,
\#\label{eq::def_Q}
Q^\pi_h(s_h, a_h; \{r_h\}_{h\in[H]}, \PP) = \sum^H_{\tau = h}\EE\bigl[ r_\tau(s_\tau, a_\tau) \biggiven s_h, a_h, \pi\bigr],\quad \forall h\in[H], ~(s_h, a_h)\in\cS\times\cA,
\#
where $a_\tau = \pi_\tau(s_\tau)$ and $s_{\tau+1}\sim \PP(\cdot \given s_\tau, a_\tau)$ for all $\tau = h, \ldots, H$. Similarly, we define the value functions $\{V^\pi_h\}_{h\in[H]}$ as follows,
\#\label{eq::def_V}
V^\pi_h(s_h; \{r_h\}_{h\in[H]}, \PP) = \sum^H_{\tau = h}\EE\bigl[ r_\tau(s_\tau, a_\tau) \biggiven s_h,  \pi\bigr],\quad \forall h\in[H], ~(s_h, a_h)\in\cS\times\cA.
\#
We define optimal policy $\pi^* = \{\pi^*_h\}_{h\in[H]}$ as the maximizer of the following optimization problem,
\#\label{def::pi}
\pi^* = \argmax_{\pi} V^\pi_1(s_1; \{r_h\}_{h\in[H]}, \PP).
\#
Correspondingly, we define $V^*$ and $Q^*$ the value and action-value functions corresponding to the optimal policy $\pi^*$.

The goal of reinforcement learning (RL) is to sequentially select the policy $\pi^k = \{\pi^k_h\}_{h\in[H]}$ based on the previous experiences, aiming to maximize the expected cumulative reward collected by the agent in the interaction process. Equivalently, the goal is to minimize the following regret,
\#\label{eq::def_regret}
R(K) = \sum^K_{k = 1} V^*(s_1) - V^{\pi^k}(s_1),
\#
where $K$ is the total number of interactions and $s_1$ is the fixed initial state. Intuitively, the regret $R(K)$ describes the deviation between the policies executed in the interaction process and the optimal policy.


\subsection{The Online Nonlinear Control Problem}
We consider the online nonlinear control problem with the following transition dynamics,
\$
s_{h+1} = f(s_h, a_h) + \epsilon, \quad \text{where}~\epsilon \sim N(0, \sigma^2\cdot I),\quad \forall h\in[H], ~(s_h, a_h)\in\cS\times\cA.
\$
Here the function $f: \cS\times\cA\mapsto \cS$ belongs to a Reproducing Kernel Hilbert Space (RKHS) with known kernel and the noise $\epsilon$ is independent across transitions. Such transition is also known as the Kernelized Nonlinear Regulator (KNR) in previous study~\citep{kakade2020information, song2021pc}. In this work, we follow~\cite{mania2022active, kakade2020information, song2021pc} and consider a primal version of such transition dynamics as the underlying transition dynamics for the RL problem, which is defined as follows,
\#\label{eq::def_KNR}
&s_{h+1} = f(s_h, a_h;W^*) + \epsilon, \quad\text{where}~\epsilon \sim N(0, \sigma^2\cdot I), \notag\\
&f(s_h, a_h; W^*) = W^* \phi(s_h, a_h), \quad \forall h\in[H], ~(s_h, a_h)\in\cS\times\cA.
\#
Here $\phi: \cS\times\cA \mapsto \RR^{d_\phi}$ is a \textit{known} feature embedding. Meanwhile, the state space $\cS\subseteq \RR^{d_\cS}$ is a subset of the Euclidean space with dimension $d_\cS$ and $W^* \in\RR^{d_{\cS}\times d_{\phi}}$ is the \textit{unknown} true parameter of the KNR transition dynamics.

Correspondingly, in the sequel, we denote by $Q^\pi(\cdot, \cdot; \{r_h\}_{h\in[H]}, W)$ and $V^\pi(\cdot; \{r_{h}\}_{h\in[H]}, W)$ the value functions of the policy $\pi$ under the reward functions $\{r_h\}_{h\in[H]}$ and the transition dynamics defined by the matrix $W\in\RR^{d_{\cS}\times d_\phi}$. For the simplicity of our analysis, we fix the following scaling of features and parameters.
\begin{assumption}[Normalized Model]
\label{asu::scal_param}
We assume that $\|\phi(s, a)\|_2 \leq 1/\sqrt{H}$ for all $(s, a)\in\cS\times\cA$. correspondingly, we assume that $\|W^*\|_2 =\cO( \sqrt{H})$, where $W^*$ is the true parameter of the KNR transition dynamics defined in \eqref{eq::def_KNR}.
\end{assumption}
Similar normalization assumptions also arises in~\cite{mania2022active}. We remark that the scaling assumptions in Assumption \ref{asu::scal_param} only affect the rate of $H$ in regret, and is imposed for the simplicity of our analysis.

\subsection{Model-based RL for Unknown Transition Dynamics}
In model-based RL, the agent optimizes the policy by iteratively fitting the transition dynamics based on the data collected, and conducting optimal planning on the fitted transition dynamics. For each iteration $k$, the model-based RL consists of the following steps.
\begin{itemize}
\item {\bf (i) Model Fitting.} In this step, the agent updates the parameter $W_k$ of transition dynamics based on the replay buffer $\cD_k$. 
\item {\bf (ii) Planning.} In this step, the agent conducts optimal planning based on the fitted parameter $W_k$ of transition dynamics. By planning with the fitted models, the agent updates the policy $\pi^k$.
\item {\bf (iii) Interaction.} In this step, the agent interacts with the environment with the policy $\pi^k$ and collects a trajectory $\iota_k = (s^k_1, a^k_1, \ldots, s^k_H, a^k_H, s^k_{H+1})$. The agent then updates the replay buffer by $\cD_{k+1} = \cD_k \cup \iota_k$.
\end{itemize}
In the sequel, we raise the following assumption, which assume that we have access to a planning oracle to handle the planning in step (ii).
\begin{assumption}[Planning Oracle]
\label{asu::oracle}
We assume that we have access to the oracle $\text{Plan}(\cdot, \cdot, \cdot)$, which returns the optimal policy $\pi = \text{Plan}(s_1, \{r_h\}_{h\in[H]}, W)$ for any input reward functions $\{r_h\}_{h\in[H]}$ and the parameter $W$ of the transition dynamics.
\end{assumption}

\begin{remark}[Remark on Sample Complexity]
In practice, the planning on fitted environment is typically handled by deep RL algorithms~\citep{pathak2017curiosity, luo2019algorithmic, pathak2019self, song2021pc} or model predictive control~\citep{williams2015model, chua2018deep, kakade2020information}. We remark that since such planning is conducted on the fitted environment, solving such planning problem does not raise concerns in the sample complexity of solvers. In contrast, such sample complexity concern is raised when interacting with the real environment in step (iii). We remark that the goal of exploration is to obtain a near-optimal policy with as few round of interactions $K$ as possible. When measured with the regret $R(K)$ defined in \eqref{eq::def_regret}, the goal of exploration is to design algorithms to attain an regret $R(K)$ that grows as slow as possible in terms of $K$.
\end{remark}

\section{Exploration with Randomized Reward}
In this section, we propose \algoname, an provably efficient and realizable algorithm for the RL problem with KNR dynamics. In the sequel, we describe the procedure of each step in the $k$-th iteration of \algoname.

\vspace{0.1in}
\noindent{\bf (i) Model Fitting.} Given the dataset $\cD_k = \{(s^\tau_h, a^\tau_h, s^\tau_{h+1})_{(h, \tau)\in[H]\times[k-1]}\}$, we fit the transition parameter $W^k$ by minimizing the prediction error of $s_{h+1}$ given $(s_h, a_h)$. Specifically, we minimize the following least-squares loss,
\#\label{eq::def_ls_loss}
W^k \leftarrow \argmin_{W \in\RR^{d_\cS\times d_\phi}}\sum^{H}_{h = 1}\sum^{k-1}_{\tau = 1}\|s^\tau_{h+1} - W \phi(s^\tau_h, a^\tau_h)\|^2_2 + \lambda \cdot \|W\|^2_{F},
\#
where we denote by $\|\cdot\|_{F}$ the matrix Frobenius norm. The optimization in \eqref{eq::def_ls_loss} has the following explicit form solution,
\#\label{eq::def_opt_soln}
W^k \leftarrow \biggl(\sum^{H}_{h = 1}\sum^{k-1}_{\tau = 1}s^\tau_{h+1} \phi(s^\tau_h, a^\tau_h)^\top\biggr)\Lambda^{-1}_{k}, \quad \Lambda_k = \sum^H_{h = 1}\sum^{k-1}_{\tau = 1}\phi(s^\tau_h, a^\tau_h)\phi(s^\tau_h, a^\tau_h)^\top + \lambda I.
\#
\vspace{0.1in}
\noindent{\bf (ii) Planning.} In the planning stage, we aim to derive a policy $\pi^k$ that interact with the environment. There are two objectives that we aim to achieve in deriving the policy $\pi^k$. (a) Firstly, the policy $\pi^k$ should properly exploit our knowledge about the environment to optimize the cumulative reward. (b) Secondly, the policy $\pi^k$ should also incorporate our uncertainty to the environment and conduct exploration to unexplored critical events. To properly balance between (a) and (b), we need to quantify our uncertainty to the environment. Such uncertainty quantification can be done by the matrix $\Lambda_k$ defined in \eqref{eq::def_opt_soln}. Specifically, it is known~\citep{} that the matrix $\Lambda_k$ defines the following confidence region
\$
\cG_k = \bigl\{W\in\RR^{d_\cS\times d_\phi}: \|(W - W^k)\Lambda^{1/2}_k\|^2_2 \leq \beta_k \bigr\}.
\$
For properly set $\beta_k$, it is known that $W^* \in \cG_k$ with high probability. Under such observation, previous attempts~\citep{jaksch2010near, kakade2020information} attain the balance between exploitation and exploration by finding the maximizer $\pi$ of $V^\pi(s_1; \{r_h\}_{h\in[H]}, W)$ for $W \in\cG_k$, which, however, is computationally intractable. To propose a computationally tractable alternative, previous empirical approaches utilizes ensemble models to estimate the epistemic uncertainty in fitting the model with finite observations. 


In this work, we investigate the approach of directly incorporating uncertainty into the reward functions. To this end, we introduce the following perturbed reward function,
\#\label{eq::perturbed_reward}
r^k_{h, \xi}(s_h, a_h) = \bigl\{r_h(s_h, a_h) + \phi(s_h, a_h)^\top \xi^k_h\bigr\}^{+}, \quad \forall h\in[H],~(s_h, a_h)\in\cS\times\cA,
\#
where the noise $\{\xi^k_h\}_{h \in[H]}$ are sampled independently from the Gaussian distribution $\xi^k_h\sim N(0, \sigma^2_{k}\cdot \Lambda^{-1}_k)$. Intuitively, such noise has larger variance in regions that are less explored by the agent, and smaller variance in regions that are well-explored. In addition, we clip the reward to ensure that the reward is positive.

Upon perturbing the reward, we update the policy by planning based on the estimated transition and perturbed reward as follows,
\#\label{eq::update_policy}
\pi^k = \{\pi^k_h\}_{h\in[H]} = \textrm{Plan}\bigl(s_1, \{r^k_{h, \xi} \}_{h\in[H]}, W^k\bigr).
\#
We remark that the reward perturbation in \eqref{eq::perturbed_reward} is conducted before planning. The perturbed reward defined in \eqref{eq::perturbed_reward} is fixed throughout the planning stage. We summarize \algoname in Algorithm \ref{algo::random_rew}.

\begin{algorithm}
	\caption{Planning with Randomized Reward}
	\label{algo::random_rew}
	\begin{algorithmic}[1] 
		\REQUIRE Dataset $\cD$, rewards $\{r_h\}_{h\in[H]}$.
		\STATE{\bf Initialization:} Set $\Lambda_1 = \lambda\cdot I$.
		\FOR{$ k = 1, 2, \ldots, K $}
		\STATE Generate a set of independent noise $\xi^k_h \sim N(0, \sigma^2_{k}\cdot \Lambda^{-1}_k)$ for all $h\in[H]$.
		\STATE Obtain the perturbed rewards
		\$
			r^k_{h, \xi}(s_h, a_h) = \{r_h(s_h, a_h) + \phi(s_h, a_h)^\top \xi^k_h\}^{+}, \quad \forall (s_h, a_h)\in\cS\times\cA, ~h\in[H].
		\$
		\STATE Obtain the policy $\pi^k$ by calling the planning oracle,
		\$
		\pi^k = \textrm{Plan}(s_1, \{r^k_{h, \xi} \}_{h\in[H]}, W^k).
		\$
		\STATE Execute $\pi^k$ to sample a trajectory $\tau^k = \{s^k_1, a^k_1, s^k_2, ..., s^k_{H}, a^k_{H}, s^k_{H+1}\}$.
		\STATE Update the dataset $\cD_k \leftarrow \cD_{k-1} \cup \tau^k$.
		\STATE Update the model and covariate matrix
		\$
		W^{k+1} &\leftarrow \argmin_{W\in\RR{\cS\times d}} \sum^H_{h=1}\sum^{k}_{\tau = 0} \|s^\tau_{h+1} - W \phi(s^\tau_{h}, a^\tau_h)\|^2_2 + \lambda\cdot \|W\|^2_2,\\
		\Lambda_{k+1} &\leftarrow \Lambda_k + \sum^H_{h=1}\phi(s^k_h, a^k_h)\phi(s^k_h, a^k_h)^\top.
		\$
		\ENDFOR
		\end{algorithmic}
\end{algorithm}

\section{Theoretical Analysis}
\label{sec::theory_KNR}
In this section, we analyze \algoname in Algorithm \ref{algo::random_rew}. Our key observations are,
\begin{itemize}
\item the reward perturbation in \algoname leads to optimistic planning for at least a constant proportion of the interactions, and
\item such partial optimism guaranteed by \algoname is sufficient for exploration.
\end{itemize}

\subsection{Partial Optimism}
\label{sec::partial_opt}
In the sequel, we show that \algoname enjoys a partial optimism. Specifically, the following lemma holds.
\begin{lemma}[Partial Optimism]
\label{lem::main_part_opt}
Under the good event $W^* \in \cG_k = \{W\in\RR^{d_\cS\times d_\phi}: \|(W - W^k)\Lambda^{1/2}_k\|^2_2 \leq \beta_k\}$, for properly selected $\sigma_k$, it holds with probability at least $\Phi(-1)$ that
\#\label{eq::partial_opt}
V^*_1\bigl(s_1; \{r_h\}_{h\in[H]}, W^*\bigr)  - V^{\pi^k}_1\bigl(s_1; \{r^k_{h, \xi}\}_{h\in[H]}, W^k\bigr)  \leq 0,
\#
where $\Phi(\cdot)$ is the cumulative distribution function of the standard Gaussian distribution.
\end{lemma}
\begin{proof}
See \S\ref{sec::pf_opt} for a detailed proof.
\end{proof}
Lemma \ref{lem::main_part_opt} ensures that at least $\Phi(-1)$ of the value function estimation in \algoname overestimates the optimal value function $V^*_1(\cdot;\{r_h\}_{h\in[H]}, W^*)$ that we wish to obtain. As a consequence, the randomized reward in \algoname guarantees that at least $\Phi(-1)$ of the trajectories contributes to exploration under the optimism principle. Intuitively, such optimism holds since, (i) on the one hand, the randomized Gaussian perturbation ensures that the perturbed reward has a sufficiently large probability to be larger than the true reward, and (ii) on the other hand, the good event $\cG_{k}$ ensures that the value functions estimated under the true model $(\{r_h\}_{h\in[H]}, W^*)$ does not deviate too much from the value function estimated under the current model $(\{r_h\}_{h\in[H]}, W^k)$ without perturbation.

\subsection{Regret Analysis}
We highlight that the optimism guarantee in Lemma \ref{lem::main_part_opt} alone does not guarantee optimal regret. To conduct reasonable exploration, in addition to optimism, we need to ensure that the overestimation induced by perturbed reward does not deviated too far away from the value functions under the true reward. In our work, we ensure such deviation guarantee by properly incorporating the uncertainty into the transition dynamics. More specifically, recall that we define the reward perturbation as follows
\$
r^k_{h, \xi}(s_h, a_h) = \bigl\{r_h(s_h, a_h) + \phi(s_h, a_h)^\top \xi^k_h\bigr\}^{+}, \quad \forall h\in[H],~(s_h, a_h)\in\cS\times\cA,
\$
where the noise $\{\xi^k_h\}_{h \in[H]}$ are sampled independently from the Gaussian distribution $\xi^k_h\sim N(0, \sigma^2_{k}\cdot \Lambda^{-1}_k)$. Such perturbation introduces the noise $\xi^k_h$, whose variance scales with the model uncertainty $\Lambda_k$.~Such reward perturbation ensures that, with a high probability, the bias in value estimation under the perturbed reward scales with the error in transition model estimation in \eqref{eq::def_opt_soln}. Thus, as long as we have reasonable model estimation, such as minimizing least-squares error in \eqref{eq::def_opt_soln}, the overestimation induced by perturbed reward is small. Specifically, the following Theorem guarantees that \algoname has an optimal regret in $K$.
\begin{theorem}
\label{thm::regret_knr}
Let $\lambda = 1$ and $\sigma_k^2 = H^3\cdot \beta_k/\sigma^2$ with $\beta_k$ specified in Appendix \ref{sec::good_event}. Under Assumptions \ref{asu::scal_param} and \ref{asu::oracle}, it holds for $K > 1/\Phi(-1)$ that
\$
\EE\bigl[R(K)\bigr] =\cO\Bigl( (d_\cS + d_\phi)^{3/2}\cdot H^{7/2}\cdot\log^{2}(K)\cdot \sqrt{K}\Bigr),
\$
where the expectation is taken with respect to the randomized reward perturbation and trajectory sampling in \algonospace.
\end{theorem}
\begin{proof}
See \S\ref{sec::pf_reg_KNR} for a detailed proof.
\end{proof}
We remark that the rate in Theorem \ref{thm::regret_knr} is information-theoretically optimal in the number of interactions $K$ with the environment~\citep{jiang2017contextual}. We remark that comparing with the optimal planning approach such as LC3~\citep{kakade2020information}, \algoname suffers from extra dependencies in $H$, $d_\phi$ and $d_\cS$, which arises due to the random perturbation of rewards.~In addition, we highlight that, comparing with PC-MLP~\citep{song2021pc}, our algorithm attains the optimal $\cO(\sqrt{K})$ dependency with respect to $K$. Such stronger sample efficiency arises as \algoname does not require extra sampling to compute policy cover matrix, which is required by PC-MLP.

\vspace{0.1in}
\noindent{\bf Exploration with Model Uncertainty.}
We remark that the high probability optimism based on Thompson sampling typically arises in the analysis of randomized value iterations for RL~\citep{russo2019worst,zanette2020frequentist}. In contrast, our work utilizes such idea for model-based exploration. To understand such counterpart in model-based exploration, we highlight that for both model-based and model-free exploration, designing provable exploration hinges on incorporating the \textit{model uncertainty} into the value functions and its corresponding policy. In value-based approaches such as LSVI-UCB~\citep{jin2020provably}, such model uncertainty is estimated via regression of target value functions on $s_{h+1}$ with respect to $(s_h, a_h)$, and is incorporated into value functions as the bonus. In model-based approaches such as UCRL and its variants~\citep{jaksch2010near, kakade2020information}, such model uncertainty is characterized by a confidence region of transition dynamics, and is incorporated into value functions via optimistic planning. In addition, for algorithms that utilizes policy cover~\citep{song2021pc, agarwal2020pc}, such model uncertainty is obtained by aggregating the visitation trajectories of current policies. Our work instantiates such idea by directly perturbing the reward functions based on model uncertainty, which serves as a primitive view of all the exploration algorithms.

\section{A Generalization with General Function Approximation}
\label{sec::gen_approx}
A key observation from the design of \algoname is that sufficient exploration is guaranteed as long as at least a fixed proportion of iterations are dedicated to exploration with optimism. To further validate such observation, we generalize \algoname by general function approximation in the sequel. We summarize the algorithm in Algorithm \ref{algo::gen_approx}. To conduct our analysis, we assume that the estimation of transition dynamics is sufficiently accurate and satisfies the following calibrated model assumption.
\begin{algorithm}
	\caption{Planning with Randomized Reward}
	\label{algo::gen_approx}
	\begin{algorithmic}[1] 
		\REQUIRE Rewards $\{r_h\}_{h\in[H]}$.
		\STATE{\bf Initialization:} Initialize buffer $\cD_0$ as an empty set. Initialize the transition dynamics $\PP^1$.
		\FOR{$ k = 1, 2, \ldots, K $}
		\STATE Generate the randomized reward $\{r^k_{h, \xi}\}_{h\in[H]}$.
		\STATE Obtain the policy $\pi^k$ by calling the planning oracle, $\pi^k = \textrm{Plan}(s_1, \{r^k_{h, \xi} \}_{h\in[H]}, \PP^k)$.
		\STATE Execute $\pi^k$ to sample a trajectory $\tau^k = \{s^k_1, a^k_1, s^k_2, ..., s^k_{H}, a^k_{H}, s^k_{H+1}\}$.
		\STATE Update the dataset $\cD_k \leftarrow \cD_{k-1} \cup \tau^k$.
		\STATE Update the transition dynamics $\PP^{k+1}$ based on the dataset $\cD_k$.
		\ENDFOR
		\end{algorithmic}
\end{algorithm}
\begin{assumption}[Calibrated model]
\label{asu::cali_model}
Let $\PP^k$ be the transition dynamics estimated in the $k$-th iteration. For all $\delta>0$ and $k\in[K]$, it holds with probability at least $1 - \delta$ that
\$
\|\PP(\cdot\given s_h, a_h) -  \PP^k(\cdot\given s_h, a_h)\|_1 \leq \beta(\delta) \cdot \iota_k(s_h, a_h), \quad \forall k\in[K], ~(s_h, a_h)\in\cS\times\cA.
\$
Meanwhile, it holds that $\iota_k \leq 1$ for all $k\in[K]$.
\end{assumption}
Here the parameter $\beta(\delta)$ characterizes the variance in concentration, which typically scales with $\log(1/\delta)$. Similar assumption also arises in the analysis under general function approximation~\citep{curi2020efficient, kidambi2021mobile}. In addition, we remark that such assumption generalizes various commonly adopted parametric models, including the linear MDP model~\citep{jin2020provably} and the KNR model~\citep{kakade2020information} we adopted in previous sections. Correspondingly, we propose the following complexity metric for the RL problems,
\#\label{eq::def_info_gain}
I_K = \max_{\{\cD_k\}_{k\in[K]}} \sum^K_{k = 1}\sum^H_{h = 1} \iota^2_k(s^k_h, a^k_h).
\#
Here the maximization is taken over all possible dataset $\{\cD_k\}_{k\in[K]}$ collected by an online learning algorithm with $|\cD_k| = H$ for all $k\in[K]$. We remark that similar complexity metric also arises in the analysis of model-based RL with general function approximations~\citep{curi2020efficient, kakade2020information, kidambi2021mobile}.

We cast the following conditions on the reward randomization that ensures sufficient exploration.
\begin{condition}[(Optimism)]
\label{cond::opt_perturb_reward}
It holds for the randomized reward function $\{r^k_{h, \xi}\}_{h\in[H]}$ that
\$
\sum^H_{h = 1}r^k_{h, \xi}(s_h, a_h) - r_h(s_h, a_h) \geq H\cdot\beta(\delta)\cdot \sum^H_{h = 1}\iota_k(s_h, a_h),
\$
which holds uniformly for all trajectories $\{(s_h, a_h)\}_{h\in[H]}$ with probability at least $p_0$.
\end{condition}
\begin{condition}[(Concentration of Rewards)]
\label{cond::centered_reward}
It holds for all $\delta' > 0$ that $|r^k_{h, \xi} - r_h| \leq C_r(\delta')\cdot \iota_k$ with probability at least $1 - \delta'$ for all $(k,h)\in[K]\times[H]$, where $\iota_k$ is defined in Assumption \ref{asu::cali_model}.
\end{condition}
\vspace{0.1in}
\noindent{\bf Intuition Behind Reward Randomization Conditions.} We remark that Conditions \ref{cond::opt_perturb_reward} and \ref{cond::centered_reward} are the key factors for the success of the randomized reward in \algonospace. On the one hand, Condition \ref{cond::opt_perturb_reward} ensures that a constant $p_0$ proportion of the evaluations results in optimistic value functions. Such optimistic value estimation further allows for exploration under the optimism principle. On the other hand, the concentration condition in Condition \ref{cond::centered_reward} ensures that with high probability, the value function estimated under the randomized reward does not deviate too much from that evaluated under the true reward.

The following Theorem upper bounds the regret of Algorithm \ref{algo::gen_approx} under Assumption \ref{asu::cali_model}.
\begin{theorem}[Regret Bound]
\label{thm::reg_gen_approx}
Under Assumption \ref{asu::cali_model}, for the randomized reward that satisfies Conditions \ref{cond::opt_perturb_reward} and \ref{cond::centered_reward}, the regret of Algorithm \ref{algo::gen_approx} is bounded as follows,
\$
\EE\bigl[R(K)\bigr] = \cO\Bigl(\textrm{Poly}\bigl(C_r(1/K), \beta(1/K), H\bigr)\cdot I_K \cdot \sqrt{K}\Bigr).
\$
\end{theorem}
\begin{proof}
See \S\ref{sec::pf_gen_approx} for the detailed proof.
\end{proof}
We remark that for commonly used model parameterization such as linear MDP and KNR, the parameter $\beta(1/K)$ typically scales with $\log(K)$. Meanwhile, for properly designed reward randomization scheme, the term $C_r(1/K)$ also scales with $\log(K)$. Thus, Theorem \ref{thm::reg_gen_approx} shows that Algorithm \ref{algo::gen_approx} has a regret bound that scales with $\tilde\cO( I_K\cdot\sqrt{K})$, which matches the previous regret bound of exploration under model-based RL~\citep{}.
\subsection{Design of Randomized Reward}
We remark that in practice, the model uncertainty $\{\iota_k\}_{k\in[K]}$ defined in Assumption \ref{asu::cali_model} can be estimated based on disagreement of ensemble models. Thus, to instantiate Algorithm \ref{algo::gen_approx}, it remains to design proper reward randomization scheme that satisfies Conditions \ref{cond::opt_perturb_reward} and \ref{cond::centered_reward}. In what follows, we present examples of such randomized rewards.

\begin{example}[Gaussian Perturbation]
\label{eg::gaussian}
Let $r^k_{h, \xi}(s^k_h, a^k_h) = r(s^k_h, a^k_h) + \xi^k_h$, where $\{\xi^k_h\}_{(k, h)\in[K]\times[H]}$ are sampled independently from the Gaussian distribution $N(0, \sigma_k\cdot \iota^2_k(s^k_h, a^k_h))$. Under regulation conditions specified in \S\ref{sec::egs_random}, the randomized reward $\{r^k_{h, \xi}\}_{(k, h)\in[K]\times[H]}$ satisfies Conditions \ref{cond::opt_perturb_reward} and \ref{cond::centered_reward}.
\end{example}

\begin{example}[Bernoulli Perturbation]
\label{eg::bernoulli}
Let $r^k_{h, \xi}(s^k_h, a^k_h) = r(s^k_h, a^k_h) + \xi^k_h\cdot \sigma'_{k}\iota_k(s^k_h, a^k_h)$, where $\xi^k_h = 1$ with probability $1/2$ and $\xi^k_h=-1$ with probability $1/2$. For the parameters $\{\sigma'_{k}\}_{(k)\in[K]}$ specified in \S\ref{sec::egs_random}, the randomized reward $\{r^k_{h, \xi}\}_{(k, h)\in[K]\times[H]}$ satisfies Conditions \ref{cond::opt_perturb_reward} and \ref{cond::centered_reward}.
\end{example}

\vspace{0.1in}
\noindent{\bf A Comparison with Bonus-based Approaches.}
We remark that our proposed randomized reward is closely related to the reward bonus for model-based RL, which arises in the recent progress of exploration under model-based RL~\citep{kidambi2021mobile, song2021pc}. Such bonus-based approaches typically estimate the model uncertainty $\{\iota_k\}_{k\in[K]}$ and then design reward bonus that incorporates such uncertainty estimation. Indeed, one may view the randomized rewards in Examples \ref{eg::gaussian} and \ref{eg::bernoulli} as a randomized generalization of such reward bonus.
Specifically, both the randomized reward and the reward bonus accomplishes exploration with optimism principle. For randomized reward, such exploration is guaranteed by the partial optimism that we investigated in \S\ref{sec::partial_opt}. In contrast, for reward bonus, such exploration is guaranteed by directly enforcing optimism with bonus.

Given the connections between the reward bonus and randomized reward, one may be prompt to ask why using randomized reward? We highlight that adding bonus is, in fact, a pessimistic approach in order to reduce the worst-case regret. Such approach enforces a bonus that deviates from the true reward, aiming to introduce a bias in value estimations to achieve minimal worst-case regret. In comparison, randomized reward allows the perturbation to be centered around the true reward, and yields a milder deviation from the true reward. In addition, as shown in our work, such perturbation also has a worst-case regret guarantee at $\tilde\cO(\sqrt{K})$ order.

\newpage

\bibliography{refs_scholar}

\newpage
\appendix

\section{Proof of Main Result}
\label{sec::pf_main_result}
In this section, we present the proofs of main results in \S\ref{sec::theory_KNR}.

\subsection{Good Events and Parameters}
In what follows, we define the following good events for the analysis.
\label{sec::good_event}
\begin{definition}[Good Events]
\label{def::event}
We define the following good events,
\$
\cG_{W^k, \textrm{good}} &= \{ \|W^* - W^k\|^2_{\Lambda_k} \leq \beta_k\},\quad
\cG_{\xi^k, \textrm{good}} = \{ \|\xi^k_h\|^2_{\Lambda^k}\leq \beta_{k, \xi}, \quad \forall h \in [H]\},\\
\cG_{\xi^k, \textrm{opt}} &= \Bigl\{ V^*_1\bigl(s_1; \{r_h\}_{h\in[H]}, W^*\bigr)  - V^{\pi^k}_1\bigl(s_1; \{r^k_{h, \xi}\}_{h\in[H]}, W^k\bigr) \leq 0\Bigr\}.
\$
Correspondingly, we further define
\$
\cG_{W, \textrm{good}} = \bigcap_{k\in[K]}\cG_{W^k, \textrm{good}}, \qquad \cG_{\xi, \textrm{good}} = \bigcap_{k\in[K]}\cG_{\xi^k, \textrm{good}}.
\$
\end{definition}
In the sequel, we follow Lemma \ref{lem::concnetrate} and set $\beta_k$ as follows
\$
\beta_k = 2\lambda\cdot \|W^*\|^2_2 + 8\sigma^2\Bigl(d_{\cS}\cdot\log(5) + 2\log(k) + \log(4) + \log\bigl(\det(\Lambda_k)/\det(\Lambda_0)\bigr)\Bigr),
\$
where $\sigma$ is the noise variance that defines the transition dynamics in \eqref{eq::def_KNR}. Correspondingly, we set the parameter $\sigma_k$ in \algoname as follows,
\$
\sigma^2_{k} =  H^3\cdot \beta_k/\sigma^2,
\$
Meanwhile, we set the parameter $\beta_{k, \xi} = 2\sigma^2_k\cdot\log(KH/\delta)$ in Definition \ref{def::event}. It thus holds that $\PP(\cG_{\xi, \textrm{good}}) \geq 1 - \delta$.

\subsection{Optimality}
In the sequel, we present the optimality analysis of \algonospace, which is inspired by~\cite{russo2019worst} and~\cite{zanette2020frequentist}.

\label{sec::pf_opt}

\begin{lemma}[Probability of Optimality]
\label{lem::opt}
Under the good event $\cG_{W^k, \textrm{good}}$, it holds with probability at least $\Phi(-1)$ that
\#
V^*_1\bigl(s_1; \{r_h\}_{h\in[H]}, W^*\bigr)  - V^{\pi^k}_1\bigl(s_1; \{r^k_{h, \xi}\}_{h\in[H]}, W^k\bigr)  \leq 0.
\#
In other words, it holds that $\PP(\cG_{\xi^k, \textrm{opt}} \given \cG_{W^k, \textrm{good}}) \geq \Phi(-1)$.
\end{lemma}
\begin{proof}
Note that
\#\label{eq::opt_1}
&   V^{\pi^k}_1\bigl(s_1; \{r^k_{h, \xi}\}_{h\in[H]}, W^k\bigr) - V^*_1\bigl(s_1; \{r_h\}_{h\in[H]}, W^*\bigr) \notag\\
&\quad\geq   V^{\pi^*}_1\bigl(s_1; \{r^k_{h, \xi}\}_{h\in[H]}, W^k\bigr)- V^*_1\bigl(s_1; \{r_h\}_{h\in[H]}, W^*\bigr),\notag\\
&\quad \geq \phi(s_1, a_1)^\top \xi^k_1 - \EE\Bigl[V^*_2\bigl(s_2; \{r_h\}_{h\in[H]}, W^*\bigr)  \,\Big|\, s_1, a_1, W^* \Bigr] \notag\\
&\quad\qquad+ \EE\Bigl[V^{\pi^*}_2\bigl(s_2; \{r^k_{h, \xi}\}_{h\in[H]}, W^k\bigr)  \,\Big|\, s_1, a_1, W^k \Bigr],
\#
where the second inequality holds since $\{r_1 + \phi^\top \xi^k_1\}^{+}  - r_1 = \max\{\phi^\top \xi^k_1, -r_1\}\geq \phi^\top \xi^k_1$. Here we denote by $\pi^*$ the optimal policy under the model $(\{r_h\}_{h\in[H]}, W^*)$ and $a_1$ the optimal action $a_1 = \pi^*(s_1)$. It further holds that,
\#\label{eq::opt_diff}
&\EE\Bigl[V^{\pi^*}_2\bigl(s_2; \{r^k_{h, \xi}\}_{h\in[H]}, W^k\bigr)  \,\Big|\, s_1, a_1, W^k \Bigr] - \EE\Bigl[V^*_2\bigl(s_2; \{r_h\}_{h\in[H]}, W^*\bigr)  \,\Big|\, s_1, a_1, W^* \Bigr]\notag\\
&\quad= \underbrace{ \EE\Bigl[V^*_2\bigl(s_2; \{r_h\}_{h\in[H]}, W^*\bigr)  \,\Big|\, s_1, a_1, W^k \Bigr] -\EE\Bigl[V^*_2\bigl(s_2; \{r_h\}_{h\in[H]}, W^*\bigr)  \,\Big|\, s_1, a_1, W^* \Bigr]}_{\textrm{(i)}}\notag\\
&\quad\qquad +  \underbrace{\EE\Bigl[V^{\pi^*}_2\bigl(s_2; \{r^k_{h, \xi}\}_{h\in[H]}, W^k\bigr) - V^*_2\bigl(s_2; \{r_h\}_{h\in[H]}, W^*\bigr)   \,\Big|\, s_1, a_1, W^k \Bigr]}_{\textrm{(ii)}}.
\#
By Lemma \ref{lem::exp_diff} and the fact that $V^*_h\leq H$ for all $h\in[H]$, we upper bound the absolute value of term (i) as follows,
\$
|\textrm{(i)}| \leq H\cdot \|(W^k- W^*)\phi(s_1, a_1)\|_2/\sigma \leq H \cdot \|W^k- W^*\|_{\Lambda_l}\cdot \|\phi(s_1, a_1)\|_{\Lambda^{-1}_k}.
\$
Thus, under the event $\cG_{W^k, \textrm{good}})$, it further holds that
\#\label{eq::opt_bound_i}
|\textrm{(i)}| \leq  \sqrt{\beta_kH^2}/\sigma \cdot \|\phi(s_1, a_1)\|_{\Lambda^{-1}_k}.
\#
By plugging \eqref{eq::opt_bound_i} into \eqref{eq::opt_diff} and further unrolling term (ii) based on similar computation in \eqref{eq::opt_1} and \eqref{eq::opt_diff}, we conclude that
\$
&V^{\pi^k}_1\bigl(s_1; \{r^k_{h, \xi}\}_{h\in[H]}, W^k\bigr) - V^*_1\bigl(s_1; \{r_h\}_{h\in[H]}, W^*\bigr)\notag\\
&\quad \geq  \sum^H_{h = 1}\EE\bigl[ \phi(s_h, a_h)^\top \xi^k_h -  \sqrt{\beta_kH^2}/\sigma \cdot \|\phi(s_h, a_h)\|_{\Lambda^{-1}_k} \biggiven s_1, \pi^*, W^k\bigr].
\$
Note that for any given trajectory $\{(s_h, a_h)\}_{h\in[H]}$, it holds that
\$
\sum^H_{h=1} \phi(s_h, a_h)^\top \xi^k_h \sim N(0, \sigma^2_{k, H}), \quad \sigma^2_{k, H} = \sigma^2_{k}\cdot \sum^H_{h = 1}\|\phi(s_h, a_h)\|_{\Lambda^{-1}_k}^2.
\$
It then holds from the setup $\sigma^2_{k} =  H^3\cdot \beta_k/\sigma^2$ and Cauchy-Schwartz inequality that
\$
\sigma_{k, H} = \sqrt{H^3\cdot \beta_k/\sigma^2\cdot \sum^H_{h = 1}\|\phi(s_h, a_h)\|_{\Lambda^{-1}_k}^2} \geq \sqrt{\beta_k H^2}/\sigma \cdot \sum^H_{h = 1}\|\phi(s_h, a_h)\|_{\Lambda^{-1}_k}.
\$
Thus, for any given trajectory $\{(s_h, a_h)\}_{h\in[H]}$, it holds with probability at least $\Phi(-1)$ that
\$
\sum^H_{h=1} \phi(s_h, a_h)^\top \xi^k_h \geq \sigma_{k, H} \geq \sqrt{\beta_kH^2}/\sigma \cdot \|\phi(s_h, a_h)\|_{\Lambda^{-1}_k}.
\$
Hence, upon taking integration with respect to the trajectory under $s_1$, $\pi^*$, $W^k$, and the good event $\cG_{W^k, \textrm{good}}$, it holds with probability at least $\Phi(-1)$ that
\$
&V^{\pi^k}_1\bigl(s_1; \{r^k_{h, \xi}\}_{h\in[H]}, W^k\bigr) - V^*_1\bigl(s_1; \{r_h\}_{h\in[H]}, W^*\bigr)\notag\\
&  \qquad\geq  \sum^H_{h = 1}\EE\bigl[ \phi(s_h, a_h)^\top \xi^k_h -  \sqrt{\beta_kH^2}/\sigma^2 \cdot \|\phi(s_h, a_h)\|_{\Lambda^{-1}_k} \biggiven s_1, \pi^*, W^k\bigr] \geq 0.
\$
Thus, we complete the proof of Lemma \ref{lem::opt}.
\end{proof}

\begin{lemma}[Optimism Bound]
\label{lem::opt_regret}
It holds for $K > 1/\Phi(-1)$ that
\$
\EE\biggl[\sum^K_{k = 1} V^*_1\bigl(s_1; \{r_h\}_{h\in[H]}, W^*\bigr) - V^{\pi^k}_1\bigl(s_1; \{r^k_{h, \xi}\}_{h\in[H]}, W^k)  \,\bigg|\, \cG_{W, \textrm{good}}, \cG_{\xi, \textrm{good}} \biggr]= \tilde \cO(\sqrt{K}).
\$
\end{lemma}
\begin{proof}
In the sequel, we set $\delta = 1/K$ in the good events defined in Definition \ref{def::event}. We fix an arbitrary $k\in[K]$ and upper bound the following difference,
\$
\Delta_k = V^*_1\bigl(s_1; \{r_h\}_{h\in[H]}, W^*\bigr) - V^{\pi^k}_1\bigl(s_1; \{r^k_{h, \xi}\}_{h\in[H]}, W^k).
\$
We construct a noise set $\{\tilde \xi^k_h\}_{h\in[H]}$, which is an identical and independent copy of the noise set $\{\xi^k_h\}_{h\in[H]}$. Correspondingly, we define the good events $\cG_{\tilde \xi, \textrm{good}}$ and $\cG_{\tilde \xi, \textrm{opt}}$ in Definition \ref{def::event}. We further define the optimal value function $V^{\tilde\pi^k}_1\bigl(s_1; \{r^k_{h, \tilde \xi}\}_{h\in[H]}, W^k)$ under the perturbed reward set $\{r_h + \phi^\top\tilde \xi^k_h\}_{h\in[H]}$ and the transition $W^k$. It thus follows from Lemma \ref{lem::opt} that
\#
\Delta_k &= V^*_1\bigl(s_1; \{r_h\}_{h\in[H]}, W^*\bigr) - V^{\pi^k}_1\bigl(s_1; \{r^k_{h, \xi}\}_{h\in[H]}, W^k)\notag\\
&\leq \EE\Bigl[ V^{\tilde\pi^k}_1\bigl(s_1; \{r^k_{h, \tilde \xi}\}_{h\in[H]}, W^k) \,\Big|\, \cG_{\tilde \xi, \textrm{good}}, \cG_{\tilde \xi, \textrm{opt}}\Bigr] - V^{\pi^k}_1\bigl(s_1; \{r^k_{h, \xi}\}_{h\in[H]}, W^k).
\#

We further define the value function corresponding to the minimal perturbation under the good event $\cG_{\xi, \textrm{good}}$ as follows,
\#\label{eq::opt_reg_def_underline}
\{\underline\xi^k_h\}_{h\in[H]} = \argmin_{\|\xi^k_h\|^2_{\Lambda_k} \leq \beta_{k, \xi}} \max_{\pi} V^\pi_1\bigl(s_1; \{r_h + \phi^\top\xi^k_h\}^{+}_{h\in[H]}, W^k\bigr).
\#
We define $r^k_{h, \underline \xi} =\{r_h + \phi^\top\underline\xi^k_h\}^{+}$ the corresponding perturbed reward. We further define $\underline \pi^k$ the corresponding optimal policy of the model $(\{r^k_{h, \underline \xi}\}_{h\in[H]}, W^k)$. Thus, under the good event $\cG_{\xi, \textrm{good}}$, we have
\#\label{eq::opt_reg_det}
\Delta_k \leq \EE\Bigl[ V^{\tilde\pi^k}_1\bigl(s_1; \{r^k_{h, \tilde \xi}\}_{h\in[H]}, W^k\bigr) - V^{\underline\pi^k}_1\bigl(s_1;\{r^k_{h, \underline \xi}\}_{h\in[H]}, W^k\bigr)\,\Big|\, \cG_{\tilde \xi, \textrm{good}}, \cG_{\tilde \xi, \textrm{opt}}\Bigr].
\#
Meanwhile, note that under the good event $\cG_{\tilde\xi^k, \textrm{good}}$, we have
\$
V^{\tilde\pi^k}_1\bigl(s_1; \{r^k_{h, \tilde \xi}\}_{h\in[H]}, W^k\bigr) \geq V^{\underline\pi^k}_1\bigl(s_1;\{r^k_{h, \underline \xi}\}_{h\in[H]}, W^k\bigr).
\$
In what follows, we write $V^{\tilde\pi^k}_1 = V^{\tilde\pi^k}_1(s_1; \{r^k_{h, \tilde \xi}\}_{h\in[H]}, W^k)$ and $V^{\underline\pi^k}_1 = V^{\underline\pi^k}_1(s_1;\{r^k_{h, \underline \xi}\}_{h\in[H]}, W^k)$ for notational simplicity. It holds that
\#\label{eq::opt_reg_1}
\EE_{\tilde \xi}\bigl[ V^{\tilde\pi^k}_1 - V^{\underline\pi^k}_1\,\big|\, \cG_{\tilde \xi, \textrm{good}}\bigr] &= \EE_{\tilde \xi}\bigl[ V^{\tilde\pi^k}_1- V^{\underline\pi^k}_1\,\big|\, \cG_{\tilde \xi^k, \textrm{good}}, \cG_{\tilde \xi^k, \textrm{opt}}\bigr] \cdot \PP(\cG_{\tilde \xi^k, \textrm{opt}}\given \cG_{\tilde \xi^k, \textrm{good}})\notag\\
&\qquad + \EE_{\tilde \xi}\bigl[ V^{\tilde\pi^k}_1- V^{\underline\pi^k}_1\,\big|\, \cG_{\tilde \xi^k, \textrm{good}}, \cG^c_{\tilde \xi^k, \textrm{opt}}\bigr] \cdot \PP(\cG^c_{\tilde \xi^k, \textrm{opt}}\given \cG_{\tilde \xi^k, \textrm{good}})\notag\\
&\geq \EE_{\tilde \xi}\bigl[ V^{\tilde\pi^k}_1 - V^{\underline\pi^k}_1\,\big|\, \cG_{\tilde \xi^k, \textrm{good}}, \cG_{\tilde \xi^k, \textrm{opt}}\bigr] \cdot \PP(\cG_{\tilde \xi^k, \textrm{opt}}\given \cG_{\tilde \xi^k, \textrm{good}}).
\#
Meanwhile, it follows from Lemma \ref{lem::opt} that, under $\cG_{W, \textrm{good}}$,
\#\label{eq::opt_reg_prob}
\PP(\cG_{\tilde \xi^k, \textrm{opt}}\given \cG_{\tilde \xi^k, \textrm{good}}) \geq \PP(\cG_{\tilde \xi^k, \textrm{opt}}\cap \cG_{\tilde \xi^k, \textrm{good}}) \geq 1 - \PP(\cG^c_{\tilde \xi^k, \textrm{opt}}) - \PP( \cG^c_{\tilde \xi^k, \textrm{good}}) \geq \Phi(-1) - \delta.
\#
By further plugging \eqref{eq::opt_reg_1} and \eqref{eq::opt_reg_prob} into \eqref{eq::opt_reg_det}, we obtain that
\#\label{eq::opt_reg_det_1}
\Delta_k &\leq  \bigl(\Phi(-1) - \delta\bigr)^{-1}\cdot \EE_{\tilde \xi}\bigl[ V^{\tilde\pi^k}_1- V^{\underline\pi^k}_1\,\big|\, \cG_{\tilde \xi^k, \textrm{good}}\bigr] \notag\\
&= \bigl(\Phi(-1) - \delta\bigr)^{-1}\cdot \EE_{ \xi}\bigl[ V^{\pi^k}_1- V^{\underline\pi^k}_1\,\big|\, \cG_{\xi^k, \textrm{good}}\bigr],
\#
where we use the fact that the noise set $\{\tilde \xi^k_h\}_{h\in[H]}$ is an identical and independent copy of the noise set $\{\xi^k_h\}_{h\in[H]}$.

It now remains to upper bound the difference $V^{\pi^k}_1- V^{\underline\pi^k}_1$ under the good events $\cG_{\xi, \good}$ and $\cG_{W, \good}$. Note that
\#\label{eq::opt_reg_diff_underline}
&V^{\pi^k}_1\bigl(s_1; \{r^k_{h, \xi}\}_{h\in[H]}, W^k\bigr)- V^{\underline\pi^k}_1\bigl(s_1; \{r^k_{h, \underline\xi}\}_{h\in[H]}, W^k\bigr)\notag\\
&\quad \leq V^{\pi^k}_1\bigl(s_1; \{r^k_{h, \xi}\}_{h\in[H]}, W^k\bigr)- V^{\pi^k}_1\bigl(s_1; \{r^k_{h, \underline\xi}\}_{h\in[H]}, W^k\bigr),
\#
which holds since $\underline\pi^k$ is optimal for the model $( \{r^k_{h, \underline\xi}\}_{h\in[H]}, W^k)$. Meanwhile, by adding and subtracting the value function $V^{\pi^k}_1(s_1; \{r_{h}\}_{h\in[H]}, W^*)$ of $\pi^k$ under the true model $(\{r_{h}\}_{h\in[H]}, W^*)$ in \eqref{eq::opt_reg_diff_underline}, we obtain that
\#\label{eq::opt_reg_diff_underline_1}
&V^{\pi^k}_1\bigl(s_1; \{r^k_{h, \xi}\}_{h\in[H]}, W^k\bigr)- V^{\underline\pi^k}_1\bigl(s_1; \{r^k_{h, \underline\xi}\}_{h\in[H]}, W^k\bigr)\notag\\
&\quad \leq \underbrace{V^{\pi^k}_1\bigl(s_1; \{r^k_{h, \xi}\}_{h\in[H]}, W^k\bigr)- V^{\pi^k}_1\bigl(s_1; \{r_{h}\}_{h\in[H]}, W^*\bigr)}_{\textrm{(iii)}}\\
&\quad\qquad+\underbrace{V^{\pi^k}_1\bigl(s_1; \{r_{h}\}_{h\in[H]}, W^*\bigr) - V^{\pi^k}_1\bigl(s_1; \{r^k_{h, \underline\xi}\}_{h\in[H]}, W^k\bigr)}_{\textrm{(iv)}}.\notag
\#
In the sequel, we upper bound terms (iii) and (iv) in \eqref{eq::opt_reg_diff_underline_1} under the good events $\cG_{\xi, \good}$ and $\cG_{W, \good}$, respectively.

\vspace{0.1in}
\noindent{\bf Upper bound of term (iii).} The upper bound is similar to that in the proof of Lemma \ref{lem::est_err}. Note that
\#\label{eq::opt_reg_iii}
\textrm{(iii)} &= V^{\pi^k}_1\bigl(s_1; \{r^k_{h, \xi}\}_{h\in[H]}, W^k\bigr)- V^{\pi^k}_1\bigl(s_1; \{r_{h}\}_{h\in[H]}, W^*\bigr)\notag\\
&=\max\{\phi(s_1, a_1)^\top\xi^k_1, -r_1(s_1, a_1)\} + \phi(s_1, a_1)^\top\xi^k_h + \EE\Bigl[V^{\pi^k}_2\bigl(s_2; \{r^k_{h, \xi}\}_{h\in[H]}, W^k\bigr) \,\Big|\, s_1, \pi^k, W^k\Bigr] \notag\\
&\qquad - \EE\Bigl[V^{\pi^k}_2\bigl(s_2; \{r_{h}\}_{h\in[H]}, W^*\bigr) \,\Big|\, s_1, \pi^k, W^*\Bigr]\notag\\
&\leq  \|\phi(s_1, a_1)\|_{\Lambda^{-1}_k} \cdot \|\xi^k_1\|_{\Lambda_k} + \Delta^k_{\xi, 1} \\
&\qquad + \EE\Bigl[V^{\pi^k}_2\bigl(s_2; \{r^k_{h, \xi}\}_{h\in[H]}, W^k\bigr) - V^{\pi^k}_2\bigl(s_2; \{r_{h}\}_{h\in[H]}, W^*\bigr) \,\Big|\, s_1, \pi^k, W^*\Bigr]\notag,
\#
where the inequality follows from Cauchy-Schwartz inequality and the fact that $r_1 \geq 0$. Here we define $a_1 = \pi^k(s_1)$ and
\$
 \Delta^k_{\xi, 1} = \EE\Bigl[V^{\pi^k}_2\bigl(s_2; \{r^k_{h, \xi}\}_{h\in[H]}, W^k\bigr) \,\Big|\, s_1, \pi^k, W^k\Bigr] - \EE\Bigl[V^{\pi^k}_2\bigl(s_2; \{r^k_{h, \xi}\}_{h\in[H]}, W^k\bigr) \,\Big|\, s_1, \pi^k, W^*\Bigr].
\$
Note that under the good event $\cG_{\xi, \good}$, it holds for all $(s_h, a_h)\in\cS\times\cA$ and $h\in[H]$ that
\$
r^k_{h, \xi}(s_h, a_h) &\leq r(s_h, a_h) + |\phi(s_h, a_h)^\top\xi^k_h| \leq 1 + \|\phi(s_h, a_h)\|_{\Lambda^{-1}_k}\cdot \|\xi^k_h\|_{\Lambda_k}\\
&\leq 1 + \sqrt{\beta_{k, \xi}/(\lambda^2 H)}.
\$
Thus, it holds under the good event $\cG_{\xi, \good}$ that
\#\label{eq::opt_reg_upper_bound_V}
0 \leq V^{\pi^k}_2\bigl(s_2; \{r^k_{h, \xi}\}_{h\in[H]}, W^k\bigr) \leq H\cdot \Bigl(1 + \sqrt{\beta_{k, \xi}/(\lambda^2 H)}\Bigr) = H + \sqrt{\beta_{k, \xi} H}/\lambda.
\#
By Lemma \ref{lem::exp_diff}, it further holds under good events $\cG_{\xi, \good}$ and $\cG_{W, \good}$ that
\#\label{eq::opt_reg_delt_k}
\Delta^k_{\xi, 1} &= \EE\Bigl[V^{\pi^k}_2\bigl(s_2; \{r^k_{h, \xi}\}_{h\in[H]}, W^k\bigr) \,\Big|\, s_1, \pi^k, W^k\Bigr] - \EE\Bigl[V^{\pi^k}_2\bigl(s_2; \{r^k_{h, \xi}\}_{h\in[H]}, W^k\bigr) \,\Big|\, s_1, \pi^k, W^*\Bigr].\notag\\
&\leq \bigl(H + \sqrt{\beta_{k, \xi} H}/\lambda\bigr)\cdot\sigma^{-1}\cdot \|(W^k - W^*)\phi(s_1, a_1)\|_2\notag\\
&\leq \bigl(H + \sqrt{\beta_{k, \xi} H}/\lambda\bigr)\cdot\sigma^{-1}\cdot \|W^k - W^*\|_{\Lambda_k}\cdot \|\phi(s_1, a_1)\|_{\Lambda^{-1}_k} \notag\\
&\leq \bigl(H + \sqrt{\beta_{k, \xi} H}/\lambda\bigr)\cdot\sigma^{-1}\cdot \sqrt{\beta_{k}} \cdot \|\phi(s_1, a_1)\|_{\Lambda^{-1}_k}.
\#
Here the first inequality follows from Lemma \ref{lem::exp_diff} and the bounds in \eqref{eq::opt_reg_upper_bound_V}, the second inequality follows from Cauchy-Schwartz inequality, and the third inequality follows from the definition of the good event $\cG_{W, \good}$ in Definition \ref{def::event}. Plugging \eqref{eq::opt_reg_delt_k} and the definition of the good event $\cG_{\xi, \good}$ into \eqref{eq::opt_reg_iii}, we obtain that
\#\label{eq::opt_reg_iii_fin1}
\textrm{(iii)}  &\leq C_k\cdot \|\phi(s_1, a_1)\|_{\Lambda^{-1}_k}\\
& \qquad+ \EE\Bigl[V^{\pi^k}_2\bigl(s_2; \{r^k_{h, \xi}\}_{h\in[H]}, W^k\bigr) - V^{\pi^k}_2\bigl(s_2; \{r_{h}\}_{h\in[H]}, W^*\bigr) \,\Big|\, s_1, \pi^k, W^*\Bigr],\notag
\#
where we define
\$
C_k =  \bigl(H + \sqrt{\beta_{k, \xi} H}/\lambda\bigr)\cdot\sigma^{-1}\cdot \sqrt{\beta_{k}} + \sqrt{\beta_{k, \xi}}.
\$
By further unrolling \eqref{eq::opt_reg_iii_fin1}, we conclude that, under the good events $\cG_{\xi, \good}$ and $\cG_{W, \good}$, we have
\#\label{eq::opt_reg_iii_fin2}
\textrm{(iii)} \leq \EE\biggl[ \sum^H_{h = 1}C_k \cdot \|\phi(s_h, a_h)\|_{\Lambda^{-1}_k}\,\bigg|\,s_1, \pi^k, W^*, \cG_{\xi^k, \good}, \cG_{W^k, \good}\biggr].
\#
\vspace{0.1in}
\noindent{\bf Upper bound of term (iv).} The upper bound of term (iv) is similar that of term (iii). Note that
\$
\textrm{(iv)} &= V^{\pi^k}_1\bigl(s_1; \{r_{h}\}_{h\in[H]}, W^*\bigr) - V^{\pi^k}_1\bigl(s_1; \{r^k_{h, \underline\xi}\}_{h\in[H]}, W^k\bigr)\notag\\
& = \min\bigl\{-\phi(s_1, a_1)^\top \underline \xi^k_h, r_1(s_1, a_1)\bigr\} + \EE\Bigl[V^{\pi^k}_2\bigl(s_2; \{r_{h}\}_{h\in[H]}, W^*\bigr) \,\Big|\, s_1, \pi^k, W^*\Bigr]\notag\\
&\qquad - \EE\Bigl[V^{\pi^k}_2\bigl(s_2; \{r^k_{h, \underline\xi}\}_{h\in[H]}, W^k\bigr) \,\Big|\, s_1, \pi^k, W^k\Bigr]\notag\\
&\leq -\phi(s_1, a_1)^\top \underline \xi^k_h + \Delta^k_{\underline \xi, 1} \notag\\
&\qquad + \EE\Bigl[V^{\pi^k}_2\bigl(s_2; \{r_{h}\}_{h\in[H]}, W^*\bigr) - V^{\pi^k}_2\bigl(s_2; \{r^k_{h, \underline\xi}\}_{h\in[H]}, W^k\bigr) \,\Big|\, s_1, \pi^k, W^*\Bigr],
\$
where we define
\$
\Delta^k_{\underline\xi, 1} = \EE\Bigl[V^{\pi^k}_2\bigl(s_2; \{r^k_{h, \xi}\}_{h\in[H]}, W^*\bigr) \,\Big|\, s_1, \pi^k, W^k\Bigr] - \EE\Bigl[V^{\pi^k}_2\bigl(s_2; \{r^k_{h, \xi}\}_{h\in[H]}, W^k\bigr) \,\Big|\, s_1, \pi^k, W^k\Bigr].
\$
By the definition of the minimal perturbation in \eqref{eq::opt_reg_def_underline}, it holds that $\|\underline \xi^k_h\|_{\Lambda_k} \leq \beta_{\xi}$ for all $(h, k)\in[H]\times[K]$. Thus, we have
\$
-\phi(s_1, a_1)^\top \underline \xi^k_h \leq \beta_{k, \xi}\cdot \|\phi(s_1, a_1)\|_{\Lambda^{-1}_k}.
\$
The rest of the computation is almost identical to that in (iii). We omit the computation for simplicity and conclude that, under good events $\cG_{\xi, \good}$ and $\cG_{W, \good}$, we have
\#\label{eq::opt_reg_iv_fin}
\textrm{(iv)} \leq \EE\biggl[ \sum^H_{h = 1}C_k \cdot \|\phi(s_h, a_h)\|_{\Lambda^{-1}_k}\,\bigg|\,s_1, \pi^k, W^*, \cG_{\xi^k, \good}, \cG_{W^k, \good}\biggr],
\#
where we define
\#\label{eq::def_c_k}
C_k = \bigl(H + \sqrt{\beta_{k, \xi} H}/\lambda\bigr)\cdot\sigma^{-1}\cdot \sqrt{\beta_{k}} + \sqrt{\beta_{k, \xi}}.
\#

Thus, by plugging \eqref{eq::opt_reg_iii_fin2} and \eqref{eq::opt_reg_iv_fin} into \eqref{eq::opt_reg_diff_underline_1}, we conclude that
\$
\EE[\Delta_k \given \cG_{W, \textrm{good}}, \cG_{\xi, \textrm{good}}] &\leq \bigl(\Phi(-1) - \delta\bigr)^{-1}\cdot \EE_{ \xi}\bigl[ V^{\pi^k}_1- V^{\underline\pi^k}_1\,\big|\,  \cG_{W^k, \textrm{good}}, \cG_{\xi^k, \textrm{good}}\bigr]\\
&\leq \bigl(\Phi(-1) - \delta\bigr)^{-1} \cdot C_{\max}\cdot \EE\biggl[ \sum^H_{h = 1}\|\phi(s^k_h, a^k_h)\|_{\Lambda^{-1}_k}\,\bigg| \cG_{W^k, \textrm{good}}, \cG_{\xi^k, \textrm{good}}\biggr],
\$
where we define
\$
C_{\max} = C_K = \bigl(H + \sqrt{\beta_{K, \xi} H}/\lambda\bigr)\cdot\sigma^{-1}\cdot \sqrt{\beta_{K}} + \sqrt{\beta_{K, \xi}} \geq C_k, \quad\forall k \in[K],
\$
and the expectation is taken with respect to the trajectories of $\pi^k$ under the transition defined by $W^*$. Recall that we set $\lambda = 1$, $\delta = 1/K$, and $\sigma_k = H^3\cdot \beta_k/\sigma^2$. Thus, under Assumption \ref{asu::scal_param}, we obtain that
\$
\beta_K &= \cO\bigl(H + d_{\cS}+\log(K) + d_{\phi}\cdot \log(K)\bigr), \\
\beta_{K, \xi} &= 2\sigma^2_K\cdot\log(K/\delta) = \cO\bigl(H^4\cdot\log(KH) + (d_\cS + d_\phi)\cdot H^3\cdot\log^2(KH)\bigr).
\$
Thus, upon computation, we have
\$
C_{\max} = \cO\bigl((d_\cS + d_\phi)\cdot H^{3}\cdot\log^{3/2}(KH)\bigr).
\$

Finally, it holds that
\$
&\EE\biggl[\sum^K_{k = 1} V^*_1\bigl(s_1; \{r_h\}_{h\in[H]}, W^*\bigr) - V^{\pi^k}_1\bigl(s_1; \{r^k_{h, \xi}\}_{h\in[H]}, W^k)  \,\bigg|\, \cG_{W, \textrm{good}}, \cG_{\xi, \textrm{good}} \biggr]\notag\\
&\qquad \leq \bigl(\Phi(-1) - 1/K\bigr)^{-1} \cdot C_{\max}\cdot \EE\biggl[ \sum^K_{k = 1}\sum^H_{h = 1}\|\phi(s^k_h, a^k_h)\|_{\Lambda^{-1}_k}\,\bigg| \cG_{W, \textrm{good}}, \cG_{\xi, \textrm{good}}\biggr]\notag\\
&\qquad \leq \bigl(\Phi(-1) - 1/K\bigr)^{-1} \cdot C_{\max}\cdot \EE\biggl[\biggl( HK\cdot \sum^K_{k = 1}\sum^H_{h = 1}\|\phi(s^k_h, a^k_h)\|^2_{\Lambda^{-1}_k} \biggr)^{1/2}\,\bigg| \cG_{W, \textrm{good}}, \cG_{\xi, \textrm{good}}\biggr]\notag\\
&\qquad = \cO\Bigl( (d_\cS + d_\phi)^{3/2}\cdot H^{7/2}\cdot\log^{2}(KH)\cdot \sqrt{K}\Bigr),
\$
where the expectation is taken with respect to the trajectories of $\{\pi^k\}_{k\in[K]}$ under the true transition defined by $W^*$, and the last inequality follows from Lemma \ref{lem::epl} and the fact that $\log\det(\Lambda_{K+1})\leq d_\phi\cdot\log K$. Thus, we complete the proof of Lemma \ref{lem::opt_regret}.
\end{proof}

\subsection{Estimation Error}
\begin{lemma}[Estimation Error Bound]
\label{lem::est_err}
It holds for $K > 1/\Phi(-1)$ that
\$
\EE\biggl[\sum^K_{k = 1} V^{\pi^k}_1\bigl(s_1; \{r^k_{h, \xi}\}_{h\in[H]}, W^k\bigr) - V^{\pi^k}_1\bigl(s_1; \{r_{h}\}_{h\in[H]}, W^*)  \,\bigg|\, \cG_{W, \textrm{good}}, \cG_{\xi, \textrm{good}} \biggr]= \tilde \cO(\sqrt{K}).
\$
\end{lemma}
\begin{proof}
Note that
\#\label{eq::est_err_eq1}
&V^{\pi^k}_1\bigl(s_1; \{r^k_{h, \xi}\}_{h\in[H]}, W^k\bigr)- V^{\pi^k}_1\bigl(s_1; \{r_{h}\}_{h\in[H]}, W^*\bigr)\notag\\
&\quad=\max\{\phi(s_1, a_1)^\top\xi^k_1, -r_1(s_1, a_1)\} + \phi(s_1, a_1)^\top\xi^k_h + \EE\Bigl[V^{\pi^k}_2\bigl(s_2; \{r^k_{h, \xi}\}_{h\in[H]}, W^k\bigr) \,\Big|\, s_1, \pi^k, W^k\Bigr] \notag\\
&\quad\qquad - \EE\Bigl[V^{\pi^k}_2\bigl(s_2; \{r_{h}\}_{h\in[H]}, W^*\bigr) \,\Big|\, s_1, \pi^k, W^*\Bigr]\notag\\
&\quad\leq  \|\phi(s_1, a_1)\|_{\Lambda^{-1}_k} \cdot \|\xi^k_1\|_{\Lambda_k} + \Delta^k_{\xi, 1} \\
&\quad\qquad + \EE\Bigl[V^{\pi^k}_2\bigl(s_2; \{r^k_{h, \xi}\}_{h\in[H]}, W^k\bigr) - V^{\pi^k}_2\bigl(s_2; \{r_{h}\}_{h\in[H]}, W^*\bigr) \,\Big|\, s_1, \pi^k, W^*\Bigr]\notag,
\#
where the inequality follows from Cauchy-Schwartz inequality and the fact that $r_1 \geq 0$. Here we define $a_1 = \pi^k(s_1)$ and
\$
 \Delta^k_{\xi, 1} = \EE\Bigl[V^{\pi^k}_2\bigl(s_2; \{r^k_{h, \xi}\}_{h\in[H]}, W^k\bigr) \,\Big|\, s_1, \pi^k, W^k\Bigr] - \EE\Bigl[V^{\pi^k}_2\bigl(s_2; \{r^k_{h, \xi}\}_{h\in[H]}, W^k\bigr) \,\Big|\, s_1, \pi^k, W^*\Bigr].
\$
Note that under the good event $\cG_{\xi, \good}$, it holds for all $(s_h, a_h)\in\cS\times\cA$ and $h\in[H]$ that
\$
r^k_{h, \xi}(s_h, a_h) &\leq r(s_h, a_h) + |\phi(s_h, a_h)^\top\xi^k_h| \leq 1 + \|\phi(s_h, a_h)\|_{\Lambda^{-1}_k}\cdot \|\xi^k_h\|_{\Lambda_k}\\
&\leq 1 + \beta_{\xi}/\lambda.
\$
Thus, it holds under the good event $\cG_{\xi, \good}$ that
\#\label{eq::est_err_eq2}
0 \leq V^{\pi^k}_2\bigl(s_2; \{r^k_{h, \xi}\}_{h\in[H]}, W^k\bigr) \leq H\cdot (1 + \beta_{\xi}/\lambda).
\#
By Lemma \ref{lem::exp_diff}, it further holds under good events $\cG_{\xi, \good}$ and $\cG_{W, \good}$ that
\#\label{eq::est_err_eq3}
\Delta^k_{\xi, 1} &= \EE\Bigl[V^{\pi^k}_2\bigl(s_2; \{r^k_{h, \xi}\}_{h\in[H]}, W^k\bigr) \,\Big|\, s_1, \pi^k, W^k\Bigr] - \EE\Bigl[V^{\pi^k}_2\bigl(s_2; \{r^k_{h, \xi}\}_{h\in[H]}, W^k\bigr) \,\Big|\, s_1, \pi^k, W^*\Bigr].\notag\\
&\leq H\cdot (1 + \beta_{\xi}/\lambda)/\sigma\cdot \|(W^k - W^*)\phi(s_1, a_1)\|_2\notag\\
&\leq H\cdot (1 + \beta_{\xi}/\lambda)/\sigma\cdot \|W^k - W^*\|_{\Lambda_k}\cdot \|\phi(s_1, a_1)\|_{\Lambda^{-1}_k} \notag\\
&\leq H\cdot (1 + \beta_{\xi}/\lambda)/\sigma \cdot \beta_{k} \cdot \|\phi(s_1, a_1)\|_{\Lambda^{-1}_k}.
\#
Here the first inequality follows from Lemma \ref{lem::exp_diff} and the bounds in \eqref{eq::est_err_eq2}, the second inequality follows from Cauchy-Schwartz inequality, and the third inequality follows from the definition of the good event $\cG_{W, \good}$ in Definition \ref{def::event}. Plugging \eqref{eq::est_err_eq3} and the definition of the good event $\cG_{\xi, \good}$ into \eqref{eq::est_err_eq1}, we obtain that
\$
&V^{\pi^k}_1\bigl(s_1; \{r^k_{h, \xi}\}_{h\in[H]}, W^k\bigr)- V^{\pi^k}_1\bigl(s_1; \{r_{h}\}_{h\in[H]}, W^*\bigr) \leq C_k\cdot \|\phi(s_1, a_1)\|_{\Lambda^{-1}_k}\\
& \qquad+ \EE\Bigl[V^{\pi^k}_2\bigl(s_2; \{r^k_{h, \xi}\}_{h\in[H]}, W^k\bigr) - V^{\pi^k}_2\bigl(s_2; \{r_{h}\}_{h\in[H]}, W^*\bigr) \,\Big|\, s_1, \pi^k, W^*\Bigr],\notag
\$
where $C_k$ is defined in \eqref{eq::def_c_k}. By further unrolling \eqref{eq::opt_reg_iii_fin1} and summing over $k\in[K]$, we conclude that
\$
&\sum^K_{k = 1}V^{\pi^k}_1\bigl(s_1; \{r^k_{h, \xi}\}_{h\in[H]}, W^k\bigr)- V^{\pi^k}_1\bigl(s_1; \{r_{h}\}_{h\in[H]}, W^*\bigr) \notag\\
&\qquad \leq \EE\biggl[\sum^K_{k = 1} \sum^H_{h = 1}C_k \cdot \|\phi(s^k_h, a^k_h)\|_{\Lambda^{-1}_k}\,\bigg|\, \cG_{W^k, \textrm{good}}, \cG_{\xi^k, \textrm{good}} \biggr],
\$
where the expectation is taken with respect to the trajectories of $\{\pi^k\}_{k\in[K]}$ under the true transition defined by $W^*$. The rest of the computation is identical to that of Lemma \ref{lem::opt_regret}. We omit the computation and conclude that
\$
\EE\biggl[\sum^K_{k = 1} V^{\pi^k}_1\bigl(s_1; \{r^k_{h, \xi}\}_{h\in[H]}, W^k\bigr) - V^{\pi^k}_1\bigl(s_1; \{r_{h}\}_{h\in[H]}, W^*)  \,\bigg|\, \cG_{W, \textrm{good}}, \cG_{\xi, \textrm{good}} \biggr]= \tilde \cO(\sqrt{K}).
\$

\end{proof}

\subsection{Regret Analysis}
\label{sec::pf_reg_KNR}
\begin{theorem}[Expected Regret Bound]
\label{thm::reg}
We set $\lambda = 1$ and $\sigma_k$ as in \S\ref{sec::good_event}. It holds for $K > 1/\Phi(-1)$ that
\$
\EE\bigl[R(K)\bigr] =\cO\Bigl( (d_\cS + d_\phi)^{3/2}\cdot H^{7/2}\cdot\log^{2}(KH)\cdot \sqrt{K}\Bigr).
\$
\end{theorem}
\begin{proof}
It holds that
\$
\EE\bigl[R(K)\bigr]  = \EE\biggl[\sum^K_{k = 1}& V^*_1\bigl(s_1; \{r_h\}_{h\in[H]}, W^*\bigr) - V^{\pi^k}_1\bigl(s_1; \{r^k_{h, \xi}\}_{h\in[H]}, W^k) \\
&\quad + V^{\pi^k}_1\bigl(s_1; \{r^k_{h, \xi}\}_{h\in[H]}, W^k\bigr) - V^{\pi^k}_1\bigl(s_1; \{r_{h}\}_{h\in[H]}, W^*) \biggr].
\$
Thus, we have
\#\label{eq::reg_bound_pf1}
\EE\bigl[R(K)\bigr]  \leq \text{Opt} + \text{Est} + V_{\max}\cdot\sum^K_{k = 1} \bigl(\PP(\cG^c_{W^k, \good})\bigr) + V_{\max}\cdot K\cdot \PP(\cG^c_{\xi, \good}),
\#
where we define
\$
\text{Opt} &= \EE\biggl[\sum^K_{k = 1} V^*_1\bigl(s_1; \{r_h\}_{h\in[H]}, W^*\bigr) - V^{\pi^k}_1\bigl(s_1; \{r^k_{h, \xi}\}_{h\in[H]}, W^k)  \,\bigg|\, \cG_{W, \textrm{good}}, \cG_{\xi, \textrm{good}} \biggr],\\
\text{Est} &= \EE\biggl[\sum^K_{k = 1} V^{\pi^k}_1\bigl(s_1; \{r^k_{h, \xi}\}_{h\in[H]}, W^k\bigr) - V^{\pi^k}_1\bigl(s_1; \{r_{h}\}_{h\in[H]}, W^*)  \,\bigg|\, \cG_{W, \textrm{good}}, \cG_{\xi, \textrm{good}} \biggr].
\$
Plugging the bounds of $\text{Opt}$ and $\text{Est}$ in Lemmas \ref{lem::opt_regret} and \ref{lem::est_err}, respectively, the fact that $V_{\max} = H$, and $\delta = 1/T$ into \eqref{eq::reg_bound_pf1}, we conclude that
\$
\EE\bigl[R(K)\bigr] =\cO\Bigl( (d_\cS + d_\phi)^{3/2}\cdot H^{7/2}\cdot\log^{2}(KH)\cdot \sqrt{K}\Bigr).
\$
Thus, we completes the proof of Theorem \ref{thm::reg}.
\end{proof}

\section{Proof of Result in \S\ref{sec::gen_approx}}
\label{sec::pf_gen_approx}
In this section, we present the proofs of results in \S\ref{sec::gen_approx}.
\subsection{Regret Analysis}
Similar to the proofs in \S\ref{sec::pf_main_result}, we define the good event $\cG_{\textrm{good}}$ as follows,
\$
 \cG_{\textrm{good}}= \bigl\{&|r_h - r^k_{h, \xi}| \leq C_r(\delta')\cdot \iota_k, \quad \|\PP(\cdot\given s, a) - \PP^k(\cdot\given s, a)\|_1 \leq \beta_k(\delta)\cdot \iota_k(s, a), \notag\\
 &\quad \forall (k, h)\in[K]\times[H], ~(s, a)\in\cS\times\cA \bigr\}.
\$
Under Assumption \ref{asu::cali_model}, it holds for the randomized reward satisfying Condition \ref{cond::opt_perturb_reward} that $\PP(\cG_{\textrm{good}}) \geq 1 - \delta' - \delta$.

\begin{lemma}[Probability of Optimality]
Under Assumptions \ref{asu::cali_model} and the good event $\cG_{\textrm{good}}$, it holds for the randomized reward satisfying Condition \ref{cond::opt_perturb_reward} that
\$
V^*_1\bigl(s_1; \{r_h\}_{h\in[H]}, \PP\bigr) \leq V^{\pi^k}_1\bigl(s_1; \{r^k_{h, \xi}\}_{h\in[H]}, \PP^k) \geq 1/2
\$
with probability at least $p_0$.
\end{lemma}
\begin{proof}
The proof is similar to that for Lemma \ref{lem::opt}. Note that
\#\label{eq::opt_prob_eq1}
&V^{\pi^k}_1\bigl(s_1; \{r^k_{h, \xi}\}_{h\in[H]}, \PP^k) - V^*_1\bigl(s_1; \{r_h\}_{h\in[H]}, \PP\bigr) \notag\\
&\qquad\geq V^{\pi^*}_1\bigl(s_1; \{r^k_{h, \xi}\}_{h\in[H]}, \PP^k) - V^*\bigl(s_1; \{r_h\}_{h\in[H]}, \PP\bigr)\notag\\
&\qquad= r^k_{h, \xi}\bigl(s_1, \pi^*(s_1)\bigr) - r_h\bigl(s_1, \pi^*(s_1)\bigr) +\EE\Bigl[V^{\pi^*}_2\bigl(s_2; \{r^k_{h, \xi}\}_{h\in[H]}, \PP^k)\,\Big|\, s_1, \pi^*, \PP^k  \Bigr]\notag\\
&\qquad\quad - \EE\Bigl[V^*_2\bigl(s_2; \{r_h\}_{h\in[H]}, \PP\bigr)\,\Big|\, s_1, \pi^*, \PP  \Bigr],
\#
where we denote by $\pi^*$ the optimal policy under the environment defined by $(\{r_h\}_{h\in[H]}, \PP)$, and the first inequality follows from the optimality of $\pi^k$ under the environment defined by $(\{r^k_{h, \xi}\}_{h\in[H]}, \PP^k)$. Meanwhile, it holds that
\#\label{eq::opt_prob_eq2}
&\EE\Bigl[V^{\pi^*}_2\bigl(s_2; \{r^k_{h, \xi}\}_{h\in[H]}, \PP^k)\,\Big|\, s_1, \pi^*, \PP^k  \Bigr] - \EE\Bigl[V^*_2\bigl(s_2; \{r_h\}_{h\in[H]}, \PP\bigr)\,\Big|\, s_1, \pi^*, \PP  \Bigr]\notag\\
&\qquad = \underbrace{\EE\Bigl[V^*_2\bigl(s_2; \{r_h\}_{h\in[H]}, \PP\bigr)\,\Big|\, s_1, \pi^*, \PP^k  \Bigr] - \EE\Bigl[V^*_2\bigl(s_2; \{r_h\}_{h\in[H]}, \PP\bigr)\,\Big|\, s_1, \pi^*, \PP  \Bigr]}_{\textrm{(i)}}\notag\\
&\qquad\quad + \underbrace{\EE\Bigl[V^{\pi^*}_2\bigl(s_2; \{r^k_{h, \xi}\}_{h\in[H]}, \PP^k) - V^*_2\bigl(s_2; \{r_h\}_{h\in[H]}, \PP\bigr)\,\Big|\, s_1, \pi^*, \PP^k  \Bigr]}_{\textrm{(ii)}}.
\#
By Assumption \ref{asu::cali_model}, H\"older's inequality, and the fact that $V^*_1\leq H$, it holds under the good event $\cG_{\textrm{good}}$ that
\$
|\textrm{(i)}| \leq \EE\bigl[ H\cdot\|(\PP - \PP^k)(\cdot\given s_1, a_1)\|_1 \biggiven s_1, \pi^*\bigr] \leq \EE\bigl[ H\cdot\beta(\delta)\cdot\iota_k(s_1, a_1) \biggiven s_1, \pi^*\bigr].
\$
By further unrolling term (ii) in \eqref{eq::opt_prob_eq2}, we conclude from \eqref{eq::opt_prob_eq1} that
\$
&V^{\pi^k}_1\bigl(s_1; \{r^k_{h, \xi}\}_{h\in[H]}, \PP^k) - V^*_1\bigl(s_1; \{r_h\}_{h\in[H]}, \PP\bigr)\notag\\
&\qquad\geq \EE\biggl[\sum^H_{h = 1} r^k_{h, \xi}(s_1, a_1) - r_h(s_1, a_1) - H\cdot\beta_k\sum^H_{h = 1}\iota_k(s_h, a_h)\,\Bigg|\, s_1, \pi^*, \PP^k \biggr].
\$
Thus, under the optimism condition in Condition \ref{cond::opt_perturb_reward}, it holds that
\$
V^{\pi^k}_1\bigl(s_1; \{r^k_{h, \xi}\}_{h\in[H]}, \PP^k) - V^*_1\bigl(s_1; \{r_h\}_{h\in[H]}, \PP\bigr) \geq 0
\$
with probability at least $p_0 - \delta$.
\end{proof}

\begin{lemma}[Optimism Bound]
\label{lem::opt_bound}
For $\delta' \leq p_0$, it holds with probability at least $1 - \delta - \delta'$ that
\$
\EE\biggl[\sum^K_{k = 1}V^*_1\bigl(s_1; \{r_h\}_{h\in[H]}, \PP\bigr) - V^{\pi^k}_1\bigl(s_1; \{r^k_{h, \xi}\}_{h\in[H]}, \PP^k) \,\bigg|\, \cG_{\textrm{good}}\biggr] =\cO\bigl(C(\delta, \delta')\cdot I_K\cdot \sqrt{K} \bigr),
\$
where we define $C(\delta, \delta')= C_r(\delta') + H\cdot(1 + C_r(\delta'))\cdot\beta(\delta)$.
\end{lemma}
\begin{proof}
The proof is similar to that of Lemma \ref{lem::opt_regret} in \S\ref{sec::pf_opt}.
We define the following minimal perturbed value function,
\$
V^{\underline\pi^k}_1 = \argmin_{\tilde r \in \cG_{\textrm{good}}} \argmax_{\pi} V^{\pi}_1\bigl(s_1; \{\tilde r_h\}_{h\in[H]}, \PP^k\bigr).
\$
Following the same computation as in the proof of Lemma \ref{lem::opt_regret}, we obtain that
\#\label{eq::opt_bound_pf}
&\EE\biggl[\sum^K_{k = 1}V^*_1\bigl(s_1; \{r_h\}_{h\in[H]}, \PP\bigr) - V^{\pi^k}_1\bigl(s_1; \{r^k_{h, \xi}\}_{h\in[H]}, \PP^k) \,\bigg|\, \cG_{\textrm{good}}\biggr]\notag\\
&\qquad\leq (p_0 - \delta)^{-1}\cdot \underbrace{\EE\biggl[\sum^K_{k = 1}V^{\pi^k}_1\bigl(s_1; \{r^k_{h, \xi}\}_{h\in[H]}, \PP^k) - V^{\underline\pi^k}_1\bigl(s_1; \{\tilde r_h\}_{h\in[H]}, \PP^k) \,\bigg|\, \cG_{\textrm{good}}\biggr]}_{\Delta_k}.
\#
By further adding and subtracting $V^{\pi^k}(s_1; \{r_h\}, \PP)$, we obtain that
\#\label{eq::iii_iv_opt_bound}
\Delta_k &= \underbrace{\EE\biggl[\sum^K_{k = 1}V^{\pi^k}_1\bigl(s_1; \{r^k_{h, \xi}\}_{h\in[H]}, \PP^k) - V^{\pi^k}(s_1; \{r_h\}, \PP) \,\bigg|\, \cG_{\textrm{good}}\biggr]}_{\textrm{(iii)}} \notag\\
&\qquad + \underbrace{\EE\biggl[\sum^K_{k = 1} V^{\pi^k}(s_1; \{r_h\}, \PP) - V^{\underline\pi^k}_1\bigl(s_1; \{\tilde r_h\}_{h\in[H]}, \PP^k) \,\bigg|\, \cG_{\textrm{good}}\biggr]}_{\textrm{(iv)}}.
\#
Thus, upon a similar computation to that in \S\ref{sec::pf_opt}, under the good event $\cG_{\textrm{good}}$, it further holds that
\$
V^{\pi^k}_1\bigl(s_1; \{r^k_{h, \xi}\}_{h\in[H]}, \PP^k) - V^{\pi^k}(s_1; \{r_h\}, \PP) \leq \EE \biggl[C(\delta, \delta')\sum^H_{h = 1}\iota_k(s_h, a_h) \,\bigg|\, s_1, \pi^k, \PP \biggr],
\$
where we define
\$
C(\delta, \delta')= C_r(\delta') + H\cdot\bigl(1 + C_r(\delta')\bigr)\cdot\beta(\delta).
\$
Here we use the fact that $\|V^{\pi^k}_h\bigl(\cdot; \{r^k_{h, \xi}\}_{h\in[H]}, \PP^k)\|_{\infty}\leq H\cdot(1 + C_r(\delta'))$ under the good event $\cG_{\textrm{good}}$. The same bound holds for term (iv) in \eqref{eq::iii_iv_opt_bound} following the fact that $\tilde r \in \cG_{\textrm{good}}$. Thus, by plugging \eqref{eq::opt_bound_pf} into \eqref{eq::opt_bound_pf}, we conclude that
\$
&\EE\biggl[\sum^K_{k = 1}V^*_1\bigl(s_1; \{r_h\}_{h\in[H]}, \PP\bigr) - V^{\pi^k}_1\bigl(s_1; \{r^k_{h, \xi}\}_{h\in[H]}, \PP^k) \,\bigg|\, \cG_{\textrm{good}}\biggr]\notag\\
&\qquad\leq C(\delta, \delta')\cdot(p_0 - \delta)^{-1}\cdot\EE \biggl[\sum^K_{k = 1}\sum^H_{h = 1}C_k\cdot \sigma_k(s_h, a_h) \,\bigg|\, \cG_{\textrm{good}} \biggr]\\
&\qquad\leq  C(\delta, \delta')\cdot(p_0 - \delta)^{-1}\cdot\sqrt{HK}\cdot \EE \Biggl[\biggl(\sum^H_{h = 1}\iota^2_k(s_h, a_h) \biggr)^{1/2}\,\Bigg|\, \cG_{\textrm{good}} \Biggr],
\$
By further plugging into the definition of $I_K$ in \eqref{eq::def_info_gain}, we obtain that
\$
\EE\biggl[\sum^K_{k = 1}V^*_1\bigl(s_1; \{r_h\}_{h\in[H]}, \PP\bigr) - V^{\pi^k}_1\bigl(s_1; \{r^k_{h, \xi}\}_{h\in[H]}, \PP^k) \,\bigg|\, \cG_{\textrm{good}}\biggr] =\cO\bigl(C(\delta, \delta')\cdot I_K\cdot \sqrt{K} \bigr),
\$
which concludes the proof of Lemma \ref{lem::opt_bound}.
\end{proof}

\begin{lemma}[Estimation Error Bound]
\label{lem::est_bound}
It holds that
\$
\EE\biggl[\sum^K_{k = 1} V^{\pi^k}_1\bigl(s_1; \{r^k_{h, \xi}\}_{h\in[H]}, \PP^k) - V^{\pi^k}_1\bigl(s_1; \{r_{h}\}_{h\in[H]}, \PP)\,\bigg|\, \cG_{\textrm{good}}\biggr]= \cO\bigl(C(\delta, \delta')\cdot I_K\cdot \sqrt{K} \bigr).
\$
\end{lemma}
\begin{proof}
Under the good event $\cG_{\textrm{good}}$, it holds that
\#\label{eq::pf_est_rollout}
&V^{\pi^k}_1\bigl(s_1; \{r^k_{h, \xi}\}_{h\in[H]}, \PP^k) - V^{\pi^k}_1\bigl(s_1; \{r_{h}\}_{h\in[H]}, \PP)\notag\\
&\qquad \leq C_r\cdot \sigma^k(s_1, a_1) +  \EE\Bigl[   V^{\pi^k}_2\bigl(s_2; \{r^k_{h, \xi}\}_{h\in[H]}, \PP^k) \,\Big|\, s_1, \pi^k, \PP^k\Bigr]\notag\\
&\qquad\quad - \EE\bigl[V^{\pi^k}_2\bigl(s_2; \{r_{h}\}_{h\in[H]}, \PP) \biggiven s_1, \pi^k, \PP\bigr]\notag\\
&\qquad \leq C_r\cdot \sigma^k(s_1, a_1) + \Delta^k_{\xi, 1}\notag\\
&\qquad\quad + \EE\bigl[  V^{\pi^k}_2\bigl(s_2; \{r^k_{h, \xi}\}_{h\in[H]}, \PP^k) - V^{\pi^k}_2\bigl(s_2; \{r_{h}\}_{h\in[H]}, \PP) \biggiven s_1, \pi^k, \PP\bigr],
\#
where we define $a_1 = \pi^k_1(s_1)$ and
\$
\Delta^k_{\xi, 1} = \EE\Bigl[   V^{\pi^k}_2\bigl(s_2; \{r^k_{h, \xi}\}_{h\in[H]}, \PP^k) \,\Big|\, s_1, \pi^k, \PP^k\Bigr] - \EE\Bigl[   V^{\pi^k}_2\bigl(s_2; \{r^k_{h, \xi}\}_{h\in[H]}, \PP^k) \,\Big|\, s_1, \pi^k, \PP\Bigr].
\$
Under the good event $\cG_{\textrm{good}}$, it holds from H\"older's inequality that
\#\label{eq::pf_est_delt}
\Delta^k_{\xi, 1} \leq H\cdot\bigl(1 + C_r(\delta')\bigr)\cdot \beta_k\cdot \iota_k(s_1, a_1).
\#
By plugging \eqref{eq::pf_est_delt} into \eqref{eq::pf_est_rollout} and further unrolling \eqref{eq::pf_est_rollout}, we obtain that, under the good event $\cG_{\textrm{good}}$,
\#\label{eq::pf_est_k}
&V^{\pi^k}_1\bigl(s_1; \{r^k_{h, \xi}\}_{h\in[H]}, \PP^k) - V^{\pi^k}_1\bigl(s_1; \{r_{h}\}_{h\in[H]}, \PP) \notag\\
&\qquad\leq C(\delta, \delta')\cdot \EE\biggl[  \sum^H_{h = 1}\iota_k(s_h, a_h) \biggiven s_1, \pi^k, \PP\biggr],
\#
where we define $
C(\delta, \delta')= C_r(\delta') + H\cdot(1 + C_r(\delta'))\cdot\beta(\delta)
$. Thus, by summing \eqref{eq::pf_est_k} over $k\in[K]$, we obtain that
\$
&\EE\biggl[\sum^K_{k = 1} V^{\pi^k}_1\bigl(s_1; \{r^k_{h, \xi}\}_{h\in[H]}, \PP^k) - V^{\pi^k}_1\bigl(s_1; \{r_{h}\}_{h\in[H]}, \PP)\,\bigg|\, \cG_{\textrm{good}}\biggr] \notag\\
&\qquad\leq C(\delta, \delta') \cdot \EE\biggl[\sum^K_{k = 1} \sum^H_{h = 1} \sigma_k(s^k_h, a^k_h)\,\bigg|\, \cG_{\textrm{good}}\biggr] \leq C(\delta, \delta')\cdot \sqrt{HK}\cdot I_K,
\$
where $I_K$ is defined in \eqref{eq::def_info_gain}. Thus, we complete the proof of Lemma \ref{lem::est_bound}.
\end{proof}

\begin{theorem}[Regret Bound]
\label{thm::reg_gen_approx}
Let $\delta = \delta' = 1/K$. Under Assumption \ref{asu::cali_model}, for the randomized reward that satisfies Conditions \ref{cond::opt_perturb_reward} and \ref{cond::centered_reward}, we have
\$
\EE\bigl[R(K)\bigr] = \cO\biggl(\Bigl(C_r(1/K) + H\cdot\bigl(1 + C_r(1/K)\bigr)\cdot\beta(1/K)\Bigr)\cdot \sqrt{H^2K}\cdot I_K\biggr).
\$
\end{theorem}
\begin{proof}
Combining Lemmas \ref{lem::opt_bound} and \ref{lem::est_bound}, it holds that
\#\label{eq::thm1}
\EE\biggl[\sum^K_{k = 1}V^*_1\bigl(s_1; \{r_h\}_{h\in[H]}, \PP\bigr) - V^{\pi^k}_1\bigl(s_1; \{r_{h}\}_{h\in[H]}, \PP^k) \,\bigg|\, \cG_{\textrm{good}}\biggr] =\cO\bigl(C(\delta, \delta')\cdot I_K\cdot \sqrt{K} \bigr).
\#
Meanwhile, under Assumption \ref{asu::cali_model}, it holds for rewards that satisfies \ref{cond::centered_reward} that $\PP(\cG_{\textrm{good}}) 1 - \delta - \delta'$. Thus, for $\delta = \delta' = 1/K$, we have
\#\label{eq::thm2}
\EE\biggl[\sum^K_{k = 1}V^*_1\bigl(s_1; \{r_h\}_{h\in[H]}, \PP\bigr) - V^{\pi^k}_1\bigl(s_1; \{r_{h}\}_{h\in[H]}, \PP^k) \,\bigg|\, \cG^c_{\textrm{good}}\biggr] \leq 2K\cdot(\delta + \delta') = 4.
\#
Combining \eqref{eq::thm1} and \eqref{eq::thm2}, it holds that
\$
\EE\bigl[R(K)\bigr] &= \cO\bigl(C(1/K, 1/K)\cdot \sqrt{H^2K}\cdot I_K\bigr) \notag\\
&= \cO\biggl(\Bigl(C_r(1/K) + H\cdot\bigl(1 + C_r(1/K)\bigr)\cdot\beta(1/K)\Bigr)\cdot \sqrt{H^2K}\cdot I_K\biggr),
\$
which concludes the proof of Theorem \ref{thm::reg_gen_approx}.
\end{proof}

\subsection{Verification of Example}
\label{sec::egs_random}
In the sequel, we verify that Examples \ref{eg::gaussian} and \ref{eg::bernoulli} satisfies Conditions \ref{cond::opt_perturb_reward} and \ref{cond::centered_reward}.

\begin{example}[Gaussian Reward]
Let
\$
r^k_{h, \xi}(s_h, a_h) \sim N\bigl(r_h(s_h, a_h), H\cdot \beta(\delta)\cdot\sigma^2_k(s_h, a_h)\bigr),
\$
which are sampled independently over $(s_h, a_h)\in\cS\times\cA$. It holds that the probability of
\$
\sum^H_{h = 1}r^k_{h, \xi}(s_h, a_h) - r_h(s_h, a_h) \geq H\cdot\beta(\delta) \sum^H_{h = 1}\sigma_k(s_h, a_h)
\$
is greater than the following event,
\$
N\biggl(0, H\cdot\sum^H_{h = 1}\sigma^2_k(s_h, a_h)\biggr) \geq \sqrt{H\cdot\beta(\delta) \sum^H_{h = 1}\sigma^2_k(s_h, a_h)},
\$
which holds with probability at least $\Phi(-1)$. In addition, if $r/\sigma_k$ are lipschitz functions of $(s, a)\in\cS\times\cA$ and $\cS\times\cA$ has a covering number $\cC_{\cS\times\cA}$ under the same metric, it further holds that $|r^k_{h, \xi} - r_h| \leq \sqrt{H\cdot\beta(\delta)\cdot\log(\cC_{\cS\times\cA}\cdot HK/\delta')}\cdot\sigma_k$ with probability at least $1 - \delta'$. Thus, Conditions \ref{cond::opt_perturb_reward} and \ref{cond::centered_reward} are satisfied.
\end{example}

\begin{example}[Bernoulli Reward]
Let
\$
r^k_{h, \eps}(s_h, a_h) \sim r(s_h, a_h) + 2\sqrt{H}\cdot \beta(\delta)\cdot\sigma_k(s_h, a_h)\cdot \epsilon_h,
\$
where $\epsilon_h = 1$  with probability $1/2$ and $\epsilon_h = -1$ with probability $1/2$. For any trajectory $\tau$, it holds from the Khintchine's inequality~\citep{veraar2010khintchine} that
\$
\sum^H_{h = 1}\sqrt{H}\cdot \beta(\delta)\cdot\sigma_k(s_h, a_h)\cdot \epsilon_h /S(\tau) \geq 1/2
\$
with probability at least $p_0 = 3/16$, where we define
\$
S(\tau) = H \sum^H_{h = 1}\beta^2(\delta)\cdot\sigma^2_k(s_h, a_h).
\$
Thus, it holds that
\$
\sum^H_{h = 1}r^k_{h, \eps}(s_h, a_h) - r(s_h, a_h) &\geq 2\sum^H_{h = 1}\sqrt{H}\cdot \beta(\delta)\cdot\sigma_k(s_h, a_h)\cdot \epsilon_h\notag\\
&\geq S(\tau) \geq \sum^H_{h = 1}\beta(\delta)\cdot\sigma_k(s_h, a_h)
\$
with probability at least $p_0 = 3/16$. In addition, it holds that $|r^k_{h, \eps}(s_h, a_h) - r(s_h, a_h)| \leq 2\sqrt{H}\cdot \beta(\delta)\cdot\sigma_k(s_h, a_h)$. Thus, Conditions \ref{cond::opt_perturb_reward} and \ref{cond::centered_reward} are satisfied.
\end{example}

\section{Auxiliary Lemma}
\begin{lemma}[Concentration of Self-normalized Process~\citep{abbasi2011improved, kakade2020information}]
\label{lem::concnetrate}
It holds for
\$
\beta_k = 2\lambda\cdot \|W^*\|^2_2 + 8\sigma^2\Bigl(d_{\cS}\cdot\log(5) + 2\log(k) + \log(4) + \log\bigl(\det(\Lambda_k)/\det(\Lambda_0)\bigr)\Bigr)
\$
that
\$
\sum^\infty_{k = 0} \PP(\cG^c_{W^k, \good}) = \sum^\infty_{k = 0} \PP\bigl(\|W^k - W^*\|^2_{\Lambda_k}\geq \beta_k\bigr) \leq 1/2.
\$
\end{lemma}
\begin{proof}
See~\cite{kakade2020information} for a detailed proof.
\end{proof}

\begin{lemma}[Expected Difference Under Two Gaussian~\citep{kakade2020information}]
\label{lem::exp_diff}
Let $z_1 \sim N(\mu_1, \sigma^2)$ and $z_2 \sim N(\mu_2, \sigma^2)$ be two Gaussian random variables. Let $g$ be a positive measurable function. It holds that
\$
\EE_{z_1\sim N(\mu_1, \sigma^2)} \bigl[g(z_1)\bigr] - \EE_{z_2\sim N(\mu_1, \sigma^2)} \bigl[g(z_2)\bigr] \leq \min\{\|\mu_1 - \mu_2\|_2/\sigma, 1\}\cdot \sqrt{\EE_{z_1\sim N(\mu_1, \sigma^2)} \bigl[g^2(z_1)\bigr]}.
\$
\end{lemma}
\begin{proof}
See~\cite{kakade2020information} for a detailed proof.
\end{proof}

\begin{lemma}[Elliptical Potential Lemma~\citep{kakade2020information}]
\label{lem::epl}
Let $\|\phi^k_h\|_2 \leq 1/\sqrt{H}$ for all $(k, h)\in[K]\times[H]$. Let $\Lambda_1  = I$ and $\Lambda_{k+1} = \Lambda_k + \sum^H_{h = 1}\phi^k_h (\phi^k_h)^\top$. It holds that
\$
\sum^{K}_{k = 1}\sum^H_{h = 1} \|\phi^k_{h}\|^2_{\Lambda^{-1}_k} \leq 2\log\bigl(\det(\Lambda_{K+1})\cdot \det(\Lambda_0)^{-1}\bigr).
\$
\end{lemma}
\begin{proof}
Note that we have $\Lambda_k \succ \Lambda_1 = I$. It thus holds that
\$
0 \leq \sum^H_{h = 1}(\phi^k_h)^\top \Lambda^{-1}_k \phi^k_h \leq \sum^H_{h = 1}\|\phi^k_h\|^2_2 \leq 1, \quad \forall k\in[K].
\$
Meanwhile, since $x\leq 2\log(1+x)$ for $x \in [0, 1]$, we have
\#\label{eq::EPL_1}
2\log\biggl(1 + \sum^H_{h=1}(\phi^k_h)^\top \Lambda^{-1}_k \phi^k_h\biggr) \geq \sum^H_{h = 1}(\phi^k_h)^\top \Lambda^{-1}_k \phi^k_h.
\#
On the other hand, it follows from matrix determinant lemma that
\$
\log \det(\Lambda_{k+1}) - \log\det(\Lambda_k) = \log\det\biggl(I + \underbrace{\Lambda^{-1/2}_k \sum^H_{h = 1}(\phi^k_h)(\phi^k_h)^\top \Lambda^{-1/2}_k}_{\Xi_k}\biggr).
\$
Let $\{\sigma_i\}_{i\in[d]}$ be the eigenvalues of the matrix $\Xi_k$. It holds that $\sigma_i > 0$ for all $i\in[d]$ and
\$
\log\det (I + \Xi_k) = \log\Pi_{i\in[d]} (1+\sigma_i) \geq \log\biggl(1 + \sum_{i\in[d]}\sigma_i\biggr) = \log\det\bigl(1 + \tr(\Xi_k)\bigr).
\$
Thus, we have
\#\label{eq::EPL_2}
\log \det(\Lambda_{k+1}) - \log\det(\Lambda_k) &= \log\det (I + \Xi_k) \geq \log\det\bigl(1 + \tr(\Xi_k)\bigr)\notag\\
&=\log\biggl(1 + \sum^H_{h = 1}\tr\bigl(\Lambda^{-1/2}_k(\phi^k_h)(\phi^k_h)^\top\Lambda^{-1/2}_k\bigr)\biggr)\notag\\
&= \log\biggl(1 + \sum^H_{h=1}(\phi^k_h)^\top \Lambda^{-1}_k \phi^k_h\biggr).
\#
Combining \eqref{eq::EPL_1} and \eqref{eq::EPL_2}, we conclude that
\$
\sum^{K}_{k = 1}\sum^H_{h = 1} \|\phi^k_{h}\|^2_{\Lambda^{-1}_k} \leq \sum^K_{k = 1} \log \det(\Lambda_{k+1}) - \log\det(\Lambda_k) =  2\log\bigl(\det(\Lambda_{K+1})\cdot \det(\Lambda_0)^{-1}\bigr),
\$
which concludes the proof of Lemma \ref{lem::exp_diff}.
\end{proof}

\end{document}